\documentclass[journal]{IEEEtran}
\usepackage{cite}
\usepackage{amsmath,amsthm,amssymb,amsfonts,mathtools}
\usepackage{algorithmic}
\usepackage{graphicx}
\usepackage{textcomp}
\usepackage{url, color}

\usepackage{bm}
\newtheorem{lemma}{Lemma}
\newtheorem{assumption}{Assumption}
\newtheorem{definition}{Definition}
\newtheorem{theorem}{Theorem}

\newtheorem{remark}{Remark}

\newcommand{\beas}{\begin{eqnarray*}}
\newcommand{\eeas}{\end{eqnarray*}}
\newcommand{\bea}{\begin{eqnarray}}
\newcommand{\eea}{\end{eqnarray}}
\newcommand{\bes}{\begin{equation*}}
\newcommand{\ees}{\end{equation*}}
\newcommand{\be}{\begin{equation}}
\newcommand{\ee}{\end{equation}}

\newcommand{\bR}{{\mathbb R}}

\DeclareMathOperator*{\argmin}{argmin}

\newcommand{\thmref}[1]{Theorem~\ref{#1}}

\newcommand{\lemref}[1]{Lemma~\ref{#1}}

\newcommand\numberthis{\addtocounter{equation}{1}\tag{\theequation}}

\newcommand{\co}[1]{{\color{black}{#1}}}
\newcommand{\com}[1]{{\color{black}{#1}}}

\begin{document}
\title{Generalized Gradient Flows with Provable Fixed-Time Convergence and Fast Evasion of Non-Degenerate Saddle Points}
\author{Mayank Baranwal, Param Budhraja, Vishal Raj, and Ashish R. Hota, \IEEEmembership{Member, IEEE}
\thanks{Mayank Baranwal is with the Tata Consultancy Services Research, Mumbai, Maharashtra 400607, India (e-mail: baranwal.mayank@tcs.com).}
\thanks{Param Budhraja is with the Department of Electrical and Computer Engineering, Boston University, Boston, MA 02215, USA (e-mail: paramb@bu.edu).}
\thanks{Vishal Raj and Ashish~R.~Hota are with the Department of Electrical Engineering, Indian Institute of Technology, Kharagpur, West Bengal 721302, India (e-mails: vishalrajroy@iitkgp.ac.in, ahota@ee.iitkgp.ac.in).}}

\maketitle

\begin{abstract}
Gradient-based first-order convex optimization algorithms find widespread applicability in a variety of domains, including machine learning tasks. Motivated by the recent advances in fixed-time stability theory of continuous-time dynamical systems, we introduce a generalized framework for designing accelerated optimization algorithms with strongest convergence guarantees that further extend to a subclass of non-convex functions. In particular, we introduce the \emph{GenFlow} algorithm and its momentum variant that provably converge to the optimal solution of objective functions satisfying the Polyak-Łojasiewicz (PL) \co{inequality in a fixed time}. Moreover\co{,} for functions that admit non-degenerate saddle-points, we show that for the proposed GenFlow algorithm, the time required to evade these saddle-points is uniformly bounded for all initial conditions. Finally, for strongly convex-strongly concave minimax problems whose optimal solution is a saddle point, a similar scheme is shown to arrive at the optimal solution again in a fixed time. The superior convergence properties of our algorithm are validated experimentally on a variety of benchmark datasets. %The time of convergence has an intuitive dependence on the hyperparameters of the optimization algorithm. The superior convergence properties of our algorithm are validated experimentally on a variety of benchmark datasets. This work highlights how connections between optimization algorithms and continuous-time dynamical systems can be leveraged to design and analyze convergence behavior of optimization algorithms.
\end{abstract}


\section{Introduction}
Consider the unconstrained optimization problem
\begin{equation} \label{optprob}
    \min_{x\in\mathbb{R}^n} f(x),
\end{equation}
where the function $f:\mathbb{R}^n \to \mathbb{R}$ is a differentiable function. A large class of algorithms have been proposed to find an optimal solution of the above problem. Due to several computationally attractive properties, the class of gradient-based first-order optimization algorithms have seen wide applicability in machine learning tasks \cite{bottou2018optimization,beck2017first}. In order to obtain faster convergence guarantees, accelerated gradient methods have been proposed and analyzed as well \cite{polyak1964heavyball,nesterov1983nag,li2020accelerated}.  

The above class of algorithms iteratively update the candidate solutions based on (local estimate of) the gradient of the objective function. Thus, the evolution of the candidate solutions assumes the form of a discrete-time difference equation. Convergence of the iterates to the optimal solution are typically analyzed from a discrete-time viewpoint \cite{bottou2018optimization,beck2017first}. 

More recently, inspired by the seminal paper \cite{wibisono2016variational}, several works have started to view the evolution of the candidate solutions from the lens of a continuous-time dynamical system whose equilibrium point coincides with the optimal solution, and convergence to the equilibrium is tied to its stability property. For instance, the simple gradient descent algorithm for minimizing a convex function $f$ can be mapped to a continuous-time dynamical system using the gradient flow $\dot{x}(t) = -\nabla f(x(t))$. Analysis of the resulting nonlinear dynamical system not only provides fresh insights into the convergence behavior of the underlying optimization algorithm \cite{wibisono2016variational,kovachki2021continuous}, but it also enables leveraging advanced results from the field of control theory \cite{khalil2002nonlinear,lessard2016analysis} to propose new algorithms with improved stability properties associated with the optimal solution leading to accelerated convergence guarantees and under more general assumptions \cite{xu2018accelerated,orvieto2019continuous,romero2020finite,budhraja2021breaking}. 

In particular, the notion of finite-time stability \cite{bhat2000} was examined in \cite{cortes2006} and \cite{romero2020finite} in the context of continuous-time gradient-flows. In this work, we focus on the recently developed notion of fixed-time stability (FxTS)~\cite{polyakov2011nonlinear}, which strengthens the notion of finite-time stability and provides fastest convergence guarantees. An equilibrium point is globally fixed-time stable if the dynamics reaches the equilibrium point in a finite amount of time which can be uniformly upper bounded for all initial conditions. Leveraging the above notion of stability, the authors in \cite{garg2021,budhraja2021breaking} propose the fixed-time stable gradient flow (FxTS-GF) scheme, given by 
\begin{equation*}
\dot{x} = -c_1\frac{\nabla f(x)}{\|\nabla f(x)\|^{\frac{p-2}{p-1}}} -c_2\frac{\nabla f(x)}{\|\nabla f(x)\|^{\frac{q-2}{q-1}}},
\end{equation*} 
with $c_1,c_2>0$, $p>2$, $q\in (1,2)$ under which the optimizer of $f$ is shown to be a fixed-time stable equilibrium when $f$ satisfies the PL-inequality \cite{karimi2016linear}. In particular, a function $f:\bR^n\to\bR$ satisfies PL-inequality with modulus $\mu>0$ if
\begin{equation}\label{eq:pl_inequality}
    \frac{1}{2}\|\nabla f(x)\|^2\geq \mu\left(f(x)-f^*\right),
\end{equation}
where $f^* := \min_{x \in \mathbb{R}^n} f(x)$ is the minimum value of the function. Note that all strongly-convex functions satisfy PL-inequality, however, the converse is not true in general. In fact, functions that satisfy the PL-inequality need not be convex. For instance, the cost function in the logistic regression problem over a compact set is not strongly-convex, nevertheless, it satisfies the PL-inequality. \co{Among non-convex problems of practical interest, recent studies~\cite{frei2021proxy, liu2022loss} demonstrate that both over-parameterized non-linear systems and neural network loss landscapes predominantly satisfy the PL condition on most (but not all) of the parameter space. This property ensures the existence of solutions and enables effective optimization through stochastic gradient descent (SGD).}  Notably, for functions satisfying PL-inequality, every stationary point is a global minimum.\footnote{\co{Although verifying the PL condition is challenging in the absence of a known globally optimal solution, Theorem~2 in~\cite{karimi2016linear} presents various implications that can potentially be utilized to verify the PL inequality. Similarly, in the context over-parameterized learning~\cite{du2018gradient, allen2019convergence}, the existence of zero-loss solutions makes it tractable to verify PL condition in important special cases.}}

The above scheme (FxTS-GF) is closely related to the class of first order algorithms with a re-scaled or normalized gradient descent \cite{levy2016ngd,murray2019ngd,wilson2019accelerating} which have been shown to escape non-degenerate saddle points\footnote{\co{A non-degenerate saddle point is a maximum in some directions and a minimum in others. See Section \ref{section:saddle} for a formal definition.}} more efficiently compared to the conventional gradient descent scheme; the latter may take exponentially long to escape certain saddle points. 

The algorithms discussed above, namely (stochastic) gradient descent (SGD)~\cite{bottou2010large}, rescaled gradient descent~\cite{wibisono2016variational,wilson2019accelerating,murray2019ngd} and fixed-time stable gradient flow~\cite{budhraja2021breaking} are characterized by constant learning rates across each dimension of the decision variable. As a result, the descent magnitude often gets dominated by the largest absolute component of the gradient vector, leading to larger descent steps in some dimensions and smaller steps in others. This poses a significant challenge in high-dimensional non-convex problems, such as the ones that arise in the context of training neural networks. In contrast, algorithms such as Adam~\cite{kingma2014adam} often exhibit improved convergence behavior due to element-wise normalization of the gradient vector resulting in uniform scaling of effective learning rates across each dimension~\cite{duchi2011adaptive, ruder2016overview}.

\subsection{Statement of Contributions}
In this work, we propose a novel continuous-time gradient flow, termed {\it GenFlow}, which includes element-wise normalization to adaptively scale the learning rates for each dimension. It generalizes element-wise scaling algorithms, such as Adam \cite{kingma2014adam}, to guarantee (accelerated) fixed-time convergence for a large class of convex as well as non-convex cost functions. We further provide bounds on escape time from non-degenerate saddle points. Our main contributions are summarized below.

\begin{enumerate}
    \item {\bf Fixed-Time and accelerated convergence for a subclass of non-convex problems:} Most existing accelerated gradient algorithms provide faster convergence guarantees only for a special class of convex functions satisfying strong convexity and/or Lipschitz-smoothness~\cite{polyak1964heavyball,nesterov2003introductory}. In contrast, we prove that under the proposed GenFlow scheme and its momentum variant, the optimal solution is fixed-time stable for the class of cost functions satisfying the PL inequality \eqref{eq:pl_inequality}, which also includes non-convex cost functions.
    
    \item {\bf Fixed-time evasion of non-degenerate saddle points:} In high-dimensional non-convex optimization problems, including problems that arise in the context of training neural networks, the proliferation of non-degenerate saddle points induces a significant challenge for optimization algorithms \cite{dauphin2014identifying}. While the gradient descent algorithm can evade saddle points and converge to the local minima \cite{lee2016gd}, the learning rate could be vanishingly slow and time to escape from the vicinity of a saddle point can be exponentially large \cite{dauphin2014identifying,murray2019ngd,du2017exp}. Recent work has shown both empirically and theoretically that GD with normalization tends to evade saddle points more efficiently \cite{levy2016ngd,murray2019ngd,ge2015saddle}. We show that the under the proposed GenFlow scheme, \com{the time required by the trajectory to escape any neighborhood of the saddle point can be upper bounded by a constant which is independent of the size of the neighborhood, which is in contrast with the results in \cite{dauphin2014identifying,murray2019ngd,du2017exp}.}
    
    %is not only capable of escaping non-degenerate saddle points, but the time of evasion is also uniformly bounded for all initial conditions.
    
    \item {\bf Accelerated convergence for \co{minimax} problems:} In addition to cost minimization problems, the class of \co{minimax} or {\it saddle point} problems arise in several settings including game theory~\cite{du1995minimax}, machine learning~\cite{goodfellow2014generative,madry2017towards} and statistics~\cite{berger2013statistical}. We extend the proposed GenFlow scheme to this class of problems and show that for a strongly convex-strongly concave objective function, a saddle point can be reached uniformly in a fixed time.
    
    \item {\bf Empirical performance:} The proposed GenFlow algorithm and its momentum variant are used for training neural networks and generative adversarial networks (GANs) and are shown to outperform state-of-the-art (SOTA) algorithms including Adam, root mean squared propagation (RMSProp), stochastic variance reduced gradient (SVRG) and fixed-time stable gradient flow (FxTS-GF).
\end{enumerate}

\co{Our proposed work introduces significant improvements and novel contributions compared to similar fixed-time convergent gradient flow schemes in the literature~\cite{garg2021,budhraja2021breaking}. We introduce adaptive gradient scaling through element-wise gradient normalization, which further accelerates convergence speed to optimal solutions when compared with FxTS-GF~\cite{garg2021,budhraja2021breaking} (see \com{Fig.~\ref{fig:Exp}} for details). We also demonstrate, arguably for the first time, that GenFlow can escape non-degenerate saddle points in a fixed time. This class of non-convex functions was not addressed in~\cite{garg2021,budhraja2021breaking}. In addition, we develop a smooth momentum variant of GenFlow which is shown to converge to optimal solutions of functions satisfying the PL inequality in a fixed time. This is in contrast to~\cite{budhraja2021breaking} where the authors assume strong convexity and Lipschitz-smoothness on the cost function to design a momentum variant of a fixed-time convergent algorithm that is discontinuous and exhibits significant chattering near the optimal solution. We further extend the scope of GenFlow to solve minimax optimization problems in a fixed time, which commonly arise in two-player games, GANs, among others.}

\subsection{Technical Preliminaries}

\subsubsection{Notation}

The set of all real numbers is denoted by $\mathbb{R}$, and the $n$-dimensional Euclidean space is denoted by $\mathbb{R}^n$. For a vector $x\in \mathbb{R}^n$, the notation $x^\top$ is used for its transpose. The $\ell_p$ norm for $p\geq 1$ is denoted by $\|\cdot\|_p$, and when $p$ is not specified $\|\cdot\|$ denotes the $\ell_2$ norm. The set of functions $f:U\to V$, where $U\subseteq \mathbb{R}^n$ and $V\subseteq \mathbb{R}^m$, that are $k$-times continuously differentiable is denoted by $C^k (U,V)$. For a function $f\in C^1 (\mathbb{R}^n ,\mathbb{R})$, $\nabla f$ denotes its gradient. The $i^{th}$ component of the gradient vector $\nabla f$ is denoted by $\nabla_i f$. The Hessian of $f\in C^2 (\mathbb{R}^n ,\mathbb{R})$ is denoted by $D^2 f$. 

The set of all eigenvalues of a matrix $A\in \mathbb{R}^{n\times n}$ is called its spectrum, and is denoted by $\sigma(A)$. The minimum absolute value among all eigenvalues of $A$ is denoted by $|\lambda|_{\min}(A)$, i.e., $|\lambda|_{\min}(A) = \min_{\lambda \in \sigma(A)} |\lambda|$. Similarly $|\lambda|_{\max}(A)$ denotes the maximum absolute value of all eigenvalues of $A$. \co{The identity matrix of size $n\times n$ is denoted by $I_n$. For a square matrix $A$, we define $|A|\coloneqq\sqrt{A^\top A}$.}

The open ball of radius $r$ around $x\in \mathbb{R}^n$ is denoted by $B_r (x)\coloneqq \{y\in\mathbb{R}^n \: : \: \|y-x\|_2 <r \}$. For $n\geq 1$, the $n$-dimensional Lebesgue measure is denoted by $\mathcal{L}^n$. 

\subsubsection{Background and Preliminary Results}

We here present a few preliminary results that enable us to establish fixed-time convergence guarantees. We first formally define fixed-time stability of an equilibrium point and then present a sufficient condition for it. 

\begin{definition}[Fixed-Time Stability \cite{polyakov2011nonlinear}]\label{def:FxTS}
    Consider the autonomous differential equation:
    \begin{equation}\label{eq:nonlin}
        \dot{x}(t) = g(x(t)), \qquad \text{with} \qquad g(\text{\co{$x^*$}})=0.
    \end{equation}
    The \co{equilibrium $x^*$} of \eqref{eq:nonlin} is defined to be globally \emph{fixed-time stable} (FxTS), if it is Lyapunov stable, and there exists a uniform settling time $\bar{T}<\infty$, such that for any initial condition $x(0) \in \mathbb{R}^n$, we have 
    \begin{align*}
        x(t) = x^* \quad \text{for all} \quad t\geq\bar{T}.
    \end{align*}
\end{definition}

\begin{lemma}[Fixed-time stability \cite{polyakov2011nonlinear}]\label{lem:FxTS}
    Let $V:\bR^n\to\bR$ with $V\in C^1(\mathbb{R}^n,\mathbb{R})$ be a continuously differentiable Lyapunov function with $V(\text{\co{$x^*$}})=0$, $V(x)>0$ for all $x\in\bR^n\setminus\{x^*\}$, and
    \[\dot{V}(x(t)) \leq -aV(x(t))^{\gamma_1} -bV(x(t))^{\gamma_2},\]
    with $a,b>0$, $\gamma_1\in(0,1)$ and $\gamma_2>1$, then the \co{equilibrium $x^*$} of \eqref{eq:nonlin} is globally FxTS with
    \[\bar{T}\leq \frac{1}{a(1-\gamma_1)} + \frac{1}{b(\gamma_2-1)}.\]
\end{lemma}

\begin{remark}
    \lemref{lem:FxTS} not only provides a sufficient condition for an equilibrium point to be fixed-time stable, but also provides an upper-bound on the uniform settling-time $\bar{T}$ in terms of the parameters $a, b, \gamma_1, \gamma_2$. Thus, given the total budget on the time required to arrive at the equilibrium point of \eqref{eq:nonlin}, the parameters can potentially be tuned so that convergence is guaranteed within the prescribed budget. 
\end{remark}

Our main convergence theorem relies on the following power inequality lemma which is a direct consequence of H\"{o}lder's inequality~\cite{zuo2016distributed}.

\begin{lemma}\label{lem:power-ineq}
    \co{For a vector $z~=~\left(z_1,z_2,\dots,z_n\right)\in\mathbb{R}^n$}, the following holds:
    \begin{equation}
        \left|\sum\limits_{i=1}^nz_i\right|^p \leq \left\{\begin{array}{ll}
             \sum\limits_{i=1}^n |z_i|^p, & p\in(0,1], \\
             n^{p-1}\sum\limits_{i=1}^n |z_i|^p, & p>1.
        \end{array}\right.
    \end{equation}
\end{lemma}

\section{Fixed-Time Convergence of GenFlow}

Motivated by the adaptive scaling properties of algorithms with element-wise normalization, and fixed-time stability properties of dynamical systems, we propose the following gradient flow scheme, termed as {\it GenFlow}, for the problem \eqref{optprob}: 
\begin{equation} \label{eq: fadagrad}
    \dot{x}_i = -\frac{\nabla_i f(x)}{|\nabla_i f(x)|^{\frac{p-2}{p-1}}} -\frac{\nabla_i f(x)}{|\nabla_i f(x)|^{\frac{q-2}{q-1}}} \quad \text{for } i=1,2,...,n,
\end{equation}
where $p>2$ and $q\in(1,2)$. Unlike the fixed-time convergent flow described in~\cite{budhraja2021breaking}, the proposed GenFlow scheme adaptively scales the learning rate for each dimension through an element-wise normalization. We now establish convergence guarantees of GenFlow for minimization of functions that satisfy the PL-inequality \eqref{eq:pl_inequality}.

\begin{theorem}[GenFlow Fixed-Time Convergence]\label{thm:F-AdaGrad}
    Let $f:\bR^n\to\bR$ be a continuously differentiable function \co{that satisfies the PL inequality \eqref{eq:pl_inequality} and possesses a unique minimizer}. Then, the GenFlow scheme defined by \eqref{eq: fadagrad} converges to the optimal solution of $f$ in a fixed time independent of the initialization.
\end{theorem}
\begin{proof}
%[Proof of Theorem~\ref{thm:F-AdaGrad}]
We begin by considering the following candidate Lyapunov function:
    \begin{align}\label{eq:V-PL}
        V(x) \coloneqq f(x)-f^*,
    \end{align}
where $f^* = \min_{x \in \mathbb{R}^n} f(x)$. Let $x^* = \argmin_{x \in \mathbb{R}^n} f(x)$ and $f(x^*)=f^*$. Notice that the Lyapunov function is positive definite, i.e., $V(x)>0$ for all $x\neq x^*$ \co{with $V(x^*)=0$}. The choice of Lyapunov function is also motivated by the definition of PL-inequality. Taking the time-derivative of $V$ along the trajectories of~\eqref{eq: fadagrad} yields:
\begin{align}
    \dot{V} &= \sum\limits_{i=1}^n\nabla_if(x)\dot{x}_i \nonumber \\
    &= -\sum\limits_{i=1}^n|\nabla_if(x)|^{2\cdot\frac{p}{2(p-1)}} - \sum\limits_{i=1}^n|\nabla_if(x)|^{2\cdot\frac{q}{2(q-1)}}, \nonumber \\
    &\leq -\!\left(\!\sum\limits_{i=1}^n(\nabla_if(x))^2\!\!\right)^{\!\frac{p}{2(p-1)}} \!\!\!\!\!\!\!\!-n^{\frac{q-2}{2(q-1)}}\!\!\left(\!\sum\limits_{i=1}^n(\nabla_if(x))^2\!\!\right)^{\!\frac{q}{2(q-1)}} \nonumber \\
    & \leq -\left(2\mu V\right)^{\frac{p}{2(p-1)}} -n^{\frac{q-2}{2(q-1)}}\left(2\mu V\right)^{\frac{q}{2(q-1)}}, \nonumber
\end{align}
where the second last inequality follows from~\lemref{lem:power-ineq} and the last inequality follows from the PL-inequality. Observe that for the given choices of $p$ and $q$, we have  $\dfrac{p}{2(p-1)}<1$ and $\dfrac{q}{2(q-1)}>1$, i.e., the sufficient conditions for fixed-time convergence stated in~\lemref{lem:FxTS} are satisfied. Thus, the GenFlow scheme guarantees uniform convergence \co{of $x(t)$ to $x^*$} in a fixed time. 
\end{proof}

\begin{remark}
\co{Note that the Lyapunov inequality in \lemref{lem:FxTS} is sufficient for uniform convergence of the state trajectory $x(t)$ to $x^*$ in a fixed time, rather than the weaker notion of convergence of $f(x(t))$ to $f^*$.\footnote{\co{Interestingly, fixed-time convergence of $f(x(t))$ to $f^*$ inherently entails the fixed-time convergence of $x(t)$ to $x^*$. This, however, does not extend to other notions of stability, such as with exponential stability, where the convergence rate of $f(x(t))$ to $f^*$ does not automatically translate to the convergence rate of $x(t)$ to $x^*$.}} In Theorem~\ref{thm:F-AdaGrad}, we assume the presence of a unique optimizer $x^*$. However, it is important to note that when a function satisfies the PL inequality, every stationary point is a global minimum. Therefore, even without the requirement of a unique $x^*$, the convergence of $f(x(t))$ to $f^*$ can still be guaranteed. In such cases, the trajectory of $x(t)$ will converge to one of the stationary points.} Note that the uniform settling time can further be obtained in terms of the parameters $\mu, n, p$ and $q$ as stated in~\lemref{lem:FxTS}.
\end{remark}

\begin{remark}
\com{In practice, the scheme~\eqref{eq: fadagrad} is numerically implemented by discretizing the continuous-time dynamics. When the gradient is close to $0$, a small constant is often added to the denominators in order to avoid potential divisibility by zero. It is worth noting that this property is not unique to our algorithm, rather is present in any normalized gradient-based algorithm, such as Adam \cite{kingma2014adam} and Adagrad \cite{duchi2011adagrad}, and the inclusion of a small positive constant in the denominator is a common practice to prevent potential division by zero; see, for instance, PyTorch implementation of Adagrad \cite{torch_adagrad} and RMSprop \cite{torch_rmsprop}. Nevertheless, in most practical problems, the denominators are rarely if ever precisely zero, and the optimization process is typically terminated when the gradient norm reaches a very small value, such as $10^{-7}$, in which case adding the small constant may not even be required.}
\end{remark}

\subsection{Accelerated Convergence via Momentum: GenFlow(M)}

Despite adaptive scaling for each dimension, pathological curvature conditions may halt optimizer's progress along certain directions. In such scenarios, optimizers can be further accelerated using momentum-like updates~\cite{muehlebach2021optimization}. In light of this, we introduce the momentum variant of the GenFlow which we refer to as GenFlow(M):
\begin{align}\label{eq:FM-adagrad}
    \dot{v}_i &= \nabla_if(x) - v_i\left(\frac{1}{|v_i|^{\frac{p-2}{p-1}}}+\frac{1}{|v_i|^{\frac{q-2}{q-1}}}\right), \nonumber \\
    \dot{x}_i &= -v_i -\nabla_if(x)\left(\frac{1}{|\nabla_if(x)|^{\frac{p-2}{p-1}}}+\frac{1}{|\nabla_if(x)|^{\frac{q-2}{q-1}}}\right), \nonumber \\
    &\qquad \qquad\qquad\qquad \qquad\qquad\text{for } i=1,2,...,n,
\end{align}
with $p>2$ and $q\in(1,2)$. Below, we show that the GenFlow(M) dynamics converges to the optimal solution uniformly in a fixed amount of time.

\begin{theorem}[GenFlow(M) Fixed-Time Convergence]\label{thm:FM-AdaGrad}
    Let $f:\bR^n\to\bR$ be a continuously differentiable function \co{that satisfies the PL inequality \eqref{eq:pl_inequality} and possesses a unique minimizer}. Then, the GenFlow\co{(M)} scheme defined by \eqref{eq:FM-adagrad} converges to the optimal solution of $f$ in fixed time independent of the initialization.
\end{theorem}
\allowdisplaybreaks
\begin{proof}
%[Proof of Theorem~\ref{thm:FM-AdaGrad}]
    Note that the equilibrium point of GenFlow(M) \eqref{eq:FM-adagrad} is given by the tuple $(x,v) = (x^*,0)$. Since $f$ satisfies PL-inequality, a natural choice for the candidate Lyapunov function is:
    \begin{align}\label{eq:thm2-0}
        V(x,v) \coloneqq (f(x)-f^*) + \frac{1}{2}\sum_{i=1}^nv_i^2.
    \end{align}
    Taking the time-derivative of $V$ along the trajectories of~\eqref{eq:FM-adagrad} yields:
    \begin{align}\label{eq:thm2-1}
        \dot{V} &= \sum_{i=1}^n\nabla_if(x)\dot{x}_i + \sum_{i=1}^nv_i\dot{v}_i \nonumber \\
        &= -\sum\limits_{i=1}^n(\nabla_if(x))^{2\cdot\frac{p}{2(p-1)}}-\sum\limits_{i=1}^n(\nabla_if(x))^{2\cdot\frac{q}{2(q-1)}}\nonumber \\
        &\quad -\sum\limits_{i=1}^nv_i^{2\cdot\frac{p}{2(p-1)}}-\sum\limits_{i=1}^nv_i^{2\cdot\frac{q}{2(q-1)}} \nonumber \\
        &\leq -\left(\sum\limits_{i=1}^n(\nabla_if(x))^2\right)^{\alpha} \!\!\!\!-\left(\sum\limits_{i=1}^nv_i^2\right)^{\alpha} \nonumber \\
        &\quad -n^{\frac{q-2}{2(q-1)}}\left(\left(\sum\limits_{i=1}^n(\nabla_if(x))^2\right)^{\beta}   \!\!\!\!+\left(\sum\limits_{i=1}^nv_i^2\right)^{\beta}\right),
\end{align}
where the last inequality follows from~\lemref{lem:power-ineq} with $\alpha=\dfrac{p}{2(p-1)} < 1$ and $\beta=\dfrac{q}{2(q-1)} > 1$. Using the PL-inequality, \eqref{eq:thm2-1} can further be written as:
\begin{align}\label{eq:thm2-2}
    \dot{V} &\leq -\left(2\mu(f(x)-f^*)\right)^\alpha - \left(\sum\limits_{i=1}^n v_i^2\right)^{\alpha} \nonumber \\
    &\quad - \underbrace{n^{\frac{q-2}{2(q-1)}}}_{\kappa}\left(\left(2\mu(f(x)-f^*)\right)^\beta - \left(\sum\limits_{i=1}^n v_i^2\right)^{\beta}\right) \nonumber \\
    &\leq \!-2^\alpha\min(1,\mu^\alpha)\!\left(\!(f(x)\!-\!f^*)^\alpha+\!\left(\frac{1}{2}\sum\limits_{i=1}^n v_i^2\right)^{\!\!\alpha}\right)\nonumber\\
    &\quad -\kappa2^\beta\min(1,\mu^\beta)\!\left(\!(f(x)\!-\!f^*)^\beta+\!\left(\frac{1}{2}\sum\limits_{i=1}^n v_i^2\right)^{\!\!\beta}\right)\!,\nonumber\\
    &\leq -2^\alpha\min(1,\mu^\alpha)V^\alpha - 2\kappa\min(1,\mu^\beta)V^\beta,
\end{align}
where the last inequality is again obtained using~\lemref{lem:power-ineq}. Recall that~\eqref{eq:thm2-2} satisfies the sufficient condition for fixed-time stability stated in \lemref{lem:FxTS}, thus guaranteeing fixed-time convergence of $x(t)$ to $x^*$ under GenFlow(M).
\end{proof}

\co{Under the proposed scheme, element-wise normalization facilitates adaptive gradient step-sizes along each coordinate, while momentum-based updates accumulate momentum in the directions of previous updates to overcome oscillations caused by noisy gradients and smoothly navigate flat regions where gradients approach zero. These two approaches mutually accelerate basic gradient flows, and their combination further enhances the optimization process which we emphasize in our simulation results as well.}

\allowdisplaybreaks
\begin{remark}[Note on algorithmic implementation]\label{rem:discrete} It is important to note that while continuous-time dynamical systems viewpoint enables rigorous convergence analysis of an optimization algorithm, in practice, it is inevitable that \co{a discretized version is implemented on a digital computer}. While we consider simple Euler-discretization of GenFlow in our experiments, GenFlow(M) is discretized using step-sizes $(\beta,\eta)$ as follows:
\vspace{-0.5mm}
\begin{small}\begin{align*}
	v_i(k+1) &\approx \beta\!\left(\!\frac{v_i(k)}{|v_i(k)|^\frac{p-2}{p-1}}\!+\!\frac{v_i(k)}{|v_i(k)|^\frac{q-2}{q-1}}\!\right)\!+\!(1\!-\!\beta)\nabla_if(x(k)), \\
	x_i(k+1) &= x_i(k) - \eta\left(v_i(k+1)+\frac{\nabla_if(x(k))}{|\nabla_if(x(k))|^\frac{p-2}{p-1}}\right. \nonumber \\
	&\quad \qquad \qquad \qquad \qquad \quad \ +  \left.\frac{\nabla_if(x(k))}{|\nabla_if(x(k))|^\frac{q-2}{q-1}}\right).
\end{align*}
\end{small}
Here $\beta\in(0,1)$ can be understood as the momentum parameter, while $\eta$ represents the learning rate or step-size.

\co{The existing literature on discretization step-size for fixed-time convergent flows primarily consists of results of an existence nature rather than providing explicit bounds~\cite{garg2022}. It is worth mentioning that even for optimization algorithms such as Nesterov's accelerated gradient method or other accelerated methods, the upper bound on the step-size is typically derived based on the Lipschitz constant of the gradient of the objective function~\cite{suh2022continuous}. However, in our work, we do not assume Lipschitz smoothness, which differentiates our approach. Deriving an upper bound on the discretization step-size remains an interesting direction for future research.}
\end{remark}

\section{Saddle Point Evasion in a Fixed Time}\label{section:saddle}

As discussed earlier, classical gradient descent algorithm often slows down in the neighborhood of a saddle point as the gradient has small magnitude in this neighborhood. In this section, we show that due to the normalization term $|\nabla_i f(x)|^{-\frac{p-2}{p-1}}$, the magnitude of the vector field under the GenFlow \eqref{eq: fadagrad} scheme does not decrease as rapidly as for GD. As a result, GenFlow can evade saddle points quickly. Specifically, we show that the time required to escape a neighborhood of the saddle point can be upper bounded by a constant independent of the size of the neighborhood. We impose the following assumption for the analysis in this section.

\begin{assumption} \label{asmp: twice_diff}
The objective function $f$ belongs to the class $C^2(\mathbb{R}^n,\mathbb{R})$.
\end{assumption}

For functions satisfying the above assumption, a non-degenerate saddle point $x^*$ is a point such that $\nabla f(x^*)=0$ and $D^2f(x^*)$ is non-singular. A non-degenerate saddle point also has the property that there exists $\lambda\in \sigma(D^2 (f(x^*)))$ such that $\lambda<0$. Saddle points that satisfy these properties are referred to in the literature \cite{ge2015saddle,sun2015rideablesaddle} as strict (or rideable) saddle points. Non-degenerate saddle points also have a useful property that local improvements are always possible in their neighborhood.

\thmref{th: fadagrad_evasion} below provides an upper bound on the time spent by the trajectory under the GenFlow scheme in a neighborhood of a non-degenerate saddle point of the function $f$. The analysis is inspired by similar techniques in \cite{murray2019ngd}, and relies on the following key lemmas.

\begin{lemma}\label{lemma: norm_convert}
    Assume that function $f$ satisfies Assumption \ref{asmp: twice_diff} and let $x^*$ denote its non-degenerate saddle point. Let $H\coloneqq D^2 (f(x^*))$. Define $\Tilde{d}(x)\coloneqq\sqrt{(x-x^*)^\top |H|(x-x^*)}$ for $x\in \mathbb{R}^n$, where $|B|\coloneqq\sqrt{B^\top B}$ for some square matrix $B$. Denote $\Lambda_{\max}\coloneqq|\lambda|_{\max}(D^2f(x^*))$ and $\Lambda_{\min}\coloneqq|\lambda|_{\min}(D^2f(x^*))$. Then, for $a\geq 0$ the following relationships hold: 
    \begin{align}
        & \|x-x^*\|\leq \frac{a}{\sqrt{\Lambda_{\max}}} \Longrightarrow \Tilde{d}(x) \leq a, \quad \text{and} \label{eq: x_to_dx} \\
        & \Tilde{d}(x)\leq a \Longrightarrow \|x-x^*\|\leq \frac{a}{\sqrt{\Lambda_{\min}}}. \label{eq: dx_to_x}
    \end{align}
\end{lemma}
\begin{proof}
    The detailed proof is included in Appendix~\ref{append:A}.
\end{proof}

In $\mathbb{R}^n$, the notion of distance is usually defined using the Euclidean norm. On the other hand, \lemref{lemma: norm_convert} introduces an alternative notion of distance $\Tilde{d}(x)$, which is used later to prove \thmref{th: fadagrad_evasion}. We also discuss on the relationship between the Euclidean distance and the alternative distance $\Tilde{d}(x)$. The following lemma uses the Taylor's theorem to derive a few useful local approximations.

\begin{lemma} \label{lemma: taylor_approx}
Suppose the function $f$ satisfies Assumption \ref{asmp: twice_diff} and has a non-degenerate saddle point $x^*$. Then, there exist $k_1 > 0.5$, $k_2 <1$, such that for some neighbourhood of $x^*$, the following inequalities hold:
\begin{align}
& |f(x) - f(x^*)| \leq k_1 \Tilde{d}(x)^2, \label{eq: bound_f} \\
& \|\nabla f(x)\|_1 \geq k_2 \|H (x-x^*)\|_1 . \label{eq: taylor_2}
\end{align}
\end{lemma}  

\begin{proof}
The detailed proof is included in Appendix~\ref{append:B}.
\end{proof}

\begin{theorem}[Fast Evasion of Saddle-Points] \label{th: fadagrad_evasion}
Suppose the function $f$ satisfies Assumption~\ref{asmp: twice_diff}. Let $x^*$ denote a non-degenerate saddle point of $f$. Let $x(t)$ be the solution of GenFlow \eqref{eq: fadagrad} with initial condition $x_0$ such that the solution $x(t)$ does not converge to the saddle point $x^*$. Then, for all sufficiently small $r>0$, the time spent in $B_r(x^*)\setminus \{x^*\}$ is upper bounded as
\begin{align}\label{eq: saddle_evasion_result} 
    &\mathcal{L}^1\left(\{t\geq0:x(t)\in B_r(x^*)\setminus \{x^*\}\}\right)\leq \nonumber \\
    &\min\left\{\!n^{\frac{1}{p-1}}\frac{8k_1}{k_2^{\frac{p}{p-1}}}\frac{\Lambda_{\max}^{\frac{p-2}{2(p-1)}}}{\Lambda_{\min}^{\frac{p}{2(p-1)}}}r^{\frac{p-2}{p-1}}, n^{\frac{1}{q-1}} \frac{8k_1}{ k_2^{\frac{q}{q-1}}}\frac{\Lambda_{\max}^{\frac{q-2}{2(q-1)}}}{\Lambda_{\min}^{\frac{q}{2(q-1)}}}r^{\frac{q-2}{q-1}}\!\right\},
\end{align}
where $k_1>0.5$, $k_2<1$,  $\Lambda_{\max}=|\lambda|_{\max}(D^2f(x^*))$ and $\Lambda_{\min}=|\lambda|_{\min}(D^2f(x^*))$.
\end{theorem}
\begin{proof}
    The detailed proof is included in Appendix~\ref{append:C}.
\end{proof}

Note that in equation \eqref{eq: saddle_evasion_result} the exponent of the term $r^{\frac{p-2}{p-1}}$ is positive whereas the exponent of the term $r^{\frac{q-2}{q-1}}$ is negative. Hence as the radius $r$ of the neighbourhood $B_r (x^*)$ increases the term $r^{\frac{p-2}{p-1}}$ increases whereas the term $r^{\frac{q-2}{q-1}}$ decreases. Thus, at some finite value both the terms must be the same. Let the constant $\hat{r}$ be such that
\begin{equation} \label{eq: fxt_evasion_bound}
n^{\frac{1}{p-1}}\frac{8k_1}{k_2^{\frac{p}{p-1}}}\frac{\Lambda_{\max}^{\frac{p-2}{2(p-1)}}}{\Lambda_{\min}^{\frac{p}{2(p-1)}}}\hat{r}^{\frac{p-2}{p-1}} =  n^{\frac{1}{q-1}} \frac{8k_1}{ k_2^{\frac{q}{q-1}}}\frac{\Lambda_{\max}^{\frac{q-2}{2(q-1)}}}{\Lambda_{\min}^{\frac{q}{2(q-1)}}}\hat{r}^{\frac{q-2}{q-1}}.
\end{equation}
Then, for radius $r$ satisfying the conditions of Theorem~\ref{th: fadagrad_evasion}, it can be said that 
\begin{align*}
    &\mathcal{L}^1\left(\{t\geq0:x(t)\in B_r(x^*)\setminus \{x^*\}\}\right) \leq\\
    &\quad n^{\frac{1}{p-1}}\frac{8k_1}{k_2^{\frac{p}{p-1}}}\frac{\Lambda_{\max}^{\frac{p-2}{2(p-1)}}}{\Lambda_{\min}^{\frac{p}{2(p-1)}}}\hat{r}^{\frac{p-2}{p-1}} =  n^{\frac{1}{q-1}} \frac{8k_1}{ k_2^{\frac{q}{q-1}}}\frac{\Lambda_{\max}^{\frac{q-2}{2(q-1)}}}{\Lambda_{\min}^{\frac{q}{2(q-1)}}}\hat{r}^{\frac{q-2}{q-1}},    
\end{align*}
where $\hat{r}$ satisfies \eqref{eq: fxt_evasion_bound}. Note that the time required to escape the neighbourhood $B_r (x^*)$ is upper bounded by a constant that is independent of the size of the neighbourhood $r$. 

\begin{remark}
The above result holds true for initial conditions for which the trajectory does not converge to the saddle point. This is indeed true in most scenarios \cite{lee2016gd}.
\end{remark}

\section{Fixed-time convergent flows for minimax problems}\label{section:minmax}

The GenFlow scheme can further be extended to obtain solutions of strongly convex-strongly concave minimax problem:
\begin{align}\label{eq:minimax}
    \co{\min\limits_{x\in\bR^{n}}\max\limits_{y\in\bR^{m}}} g(x,y).
\end{align}
The function $g(\cdot,\cdot)$ is strongly convex-strongly concave with modulii $\mu_1, \mu_2$ if it is strongly convex in its first argument with modulus $\mu_1$ and strongly concave in the second argument with modulus $\mu_2$. We consider the following saddle point dynamics:
\begin{align}\label{eq:saddle-pt}
    \dot{x}_i &= -\frac{\nabla_{x_i} g(x,y)}{\|\nabla G\|^{\frac{p-2}{p-1}}} - \frac{\nabla_{x_i} g(x,y)}{\|\nabla G\|^{\frac{q-2}{q-1}}}, \quad &\text{for } i=1,2,\dots,\co{n}, \nonumber \\
    \dot{y}_j &= \frac{\nabla_{y_j} g(x,y)}{\|\nabla G\|^{\frac{p-2}{p-1}}} + \frac{\nabla_{y_j} g(x,y)}{\|\nabla G\|^{\frac{q-2}{q-1}}}, \quad &\text{for } j=1,2,\dots,\co{m},
\end{align}
with $p>2$, $q\in(1,2)$, and the term in the denominator $\nabla G(x,y) \coloneqq [\nabla_xg^\top(x,y) \quad -\nabla_yg^\top(x,y)]^\top$. \co{For solving minimax optimization problems, achieving accelerated convergence in one variable does not automatically ensure accelerated convergence in the other variable. To guarantee joint accelerated convergence, it is essential to couple the dynamics of both $x$ and $y$ which is achieved through the norm of the total gradient vector $\nabla G$.} Along with the normalization, this coupling is also crucial for imparting fixed-time convergence as shown in the following theorem.

%Note that the $x$ and $y$ dynamics in~\eqref{eq:saddle-pt} are coupled through the norm of the total gradient vector $\nabla G$.

\begin{theorem}[Fixed-Time Convergence to Saddle Point]\label{thm:saddle_pt_dyn}
Suppose the function $g:\bR^{\co{n}}\times \bR^{\co{m}}\to\bR$ in~\eqref{eq:minimax} is strongly convex-strongly concave with modulii $\mu_1$ and $\mu_2$, respectively. Then, the saddle-point dynamics described in~\eqref{eq:saddle-pt} converges to the unique optimal solution uniformly in a fixed amount of time, independent of the initialization.
\end{theorem}
\begin{proof}
    Note that the saddle point is also characterized by the first-order stationarity condition, and since the function $g$ is strongly convex-strongly concave in its arguments, the stationary point is indeed the saddle point (or the equilibrium of the saddle point dynamics~\eqref{eq:saddle-pt}). The strong convexity-strong concavity translates to the following set of conditions on the second-order partial derivatives:
    \begin{align}\label{eq:thm4-1}
        \nabla_{xx}^2g(x,y) \succcurlyeq \mu_1I_{n_1} \quad \text{and} \quad \nabla_{yy}^2g(x,y) \preccurlyeq -\mu_2I_{n_2}, 
    \end{align}
    for every $x \in \bR^n, y \in \bR^m$. In light of the stationarity conditions and the strongly convex-strongly concave assumptions, a natural choice for the candidate Lyapunov function is:
    \begin{align}\label{eq:thm4-2}
        V\coloneqq \frac{1}{2}\|\nabla_{x}g(x,y)\|^2 + \frac{1}{2}\|\nabla_{y}g(x,y)\|^2.
    \end{align}
    Taking time-derivative of $V$ along the trajectories of~\eqref{eq:saddle-pt} yields:
    \begin{align}\label{eq:thm4-3}
        \dot{V} &= \nabla_xg(x,y)^\top\nabla_{xx}^2g(x,y)\dot{x} + \nabla_yg(x,y)^\top\nabla_{yy}^2g(x,y)\dot{y} \nonumber \\
        &\quad + \underbrace{\left(\nabla_xg(x,y)^\top\nabla_{xy}^2g(x,y)\dot{y}+\nabla_yg(x,y)^\top\nabla_{xy}^2g(x,y)\dot{x}\right)}_{=0},
    \end{align}
    where the last term evaluates to zero following the saddle point dynamics. Substituting~\eqref{eq:saddle-pt} into~\eqref{eq:thm4-3} results in:
    \begin{align}\label{eq:thm4-4}
        \dot{V} &= \Big(-\nabla_xg^\top\nabla_{xx}^2g\nabla_xg  +\nabla_yg^\top\nabla_{yy}^2g\nabla_yg\Big)\Bigg[\frac{1}{\|\nabla G\|^{\frac{p-2}{p-1}}}\nonumber\\
        &\ \quad +\frac{1}{\|\nabla G\|^{\frac{q-2}{q-1}}}\Bigg] \nonumber\\
        %&\leq -\mu_1\frac{\|\nabla_xg(x,y)\|^2}{\|\nabla G\|^{\frac{p-2}{p-1}}+\|\nabla G\|^{\frac{q-2}{q-1}}} - \mu_2\frac{\|\nabla_yg(x,y)\|^2}{\|\nabla G\|^{\frac{p-2}{p-1}}+\|\nabla G\|^{\frac{q-2}{q-1}}},
        &\leq -\Big(\mu_1 \|\nabla_xg(x,y)\|^2 + \mu_2 \|\nabla_yg(x,y)\|^2\Big)\Bigg[\frac{1}{\|\nabla G\|^{\frac{p-2}{p-1}}}\nonumber
        \\ &\ \quad + \frac{1}{\|\nabla G\|^{\frac{q-2}{q-1}}}\Bigg]
    \end{align}
    where the last inequality follows from~\eqref{eq:thm4-1}. Recall that $\|\nabla G\|^2 = \|\nabla_x g(x,y)\|^2 + \|-\nabla_y g(x,y)\|^2$, resulting in:
    \begin{align}\label{eq:thm4-5}
        \dot{V} &\leq -\min(\mu_1,\mu_2)\left(\|\nabla G\|^{2\cdot\frac{p}{2(p-1)}} + \|\nabla G\|^{2\cdot\frac{q}{2(q-1)}}\right) \nonumber \\
        & = -\min(\mu_1,\mu_2)\left((2V)^{\frac{p}{2(p-1)}}+(2V)^{\frac{q}{2(q-1)}}\right),
    \end{align}
    which satisfies the sufficient condition for fixed-time convergence stated in \lemref{lem:FxTS}.
\end{proof}

%%%%%%%%%%%%%%%%%%%%%%%%%%%%%%%%%%%%%%%%%%%%%%%
%%%%%%%%%%%%%%%%%%%%%%%%%%%%%%%%%%%%%%%%%%%%%%%
%%%%%%%%%%%%%%%%%%%%%%%%%%%%%%%%%%%%%%%%%%%%%%%

\section{Experiments}
\begin{figure*}[!ht]
	\begin{center}
		\begin{tabular}{ccc}
			\includegraphics[width=0.6\columnwidth]{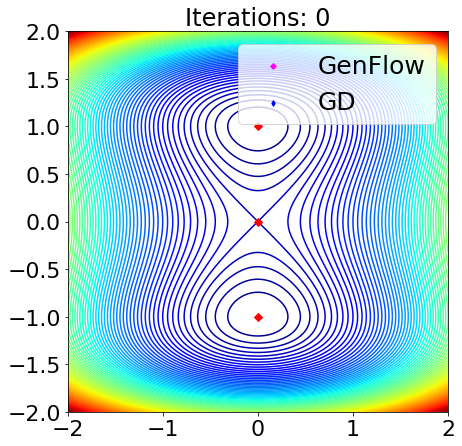}&
			\includegraphics[width=0.6\columnwidth]{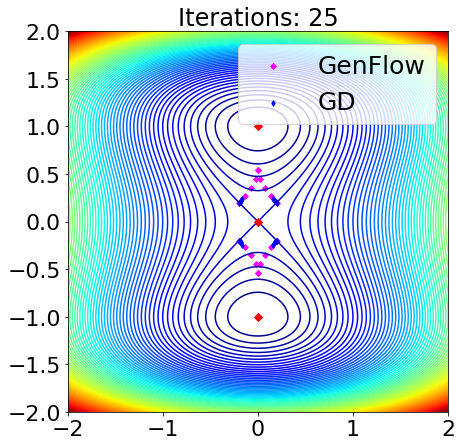}&
			\includegraphics[width=0.6\columnwidth]{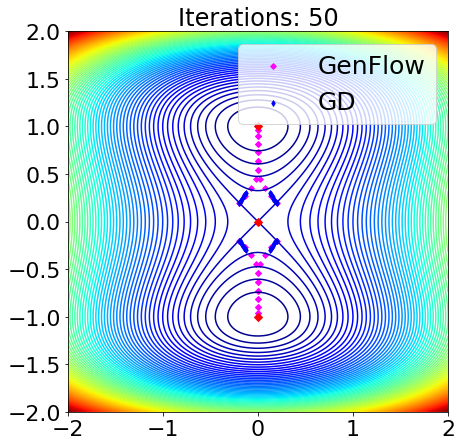}\cr
			(a)&(b)&(c)\cr
			\includegraphics[width=0.6\columnwidth]{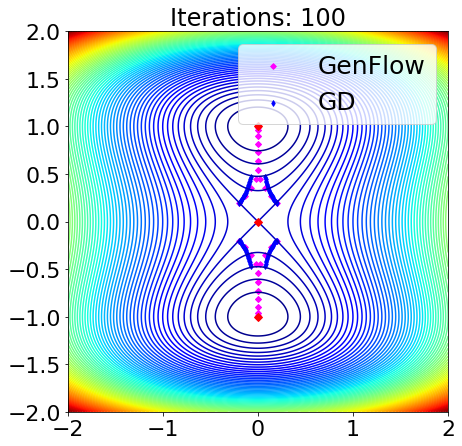}&
			\includegraphics[width=0.6\columnwidth]{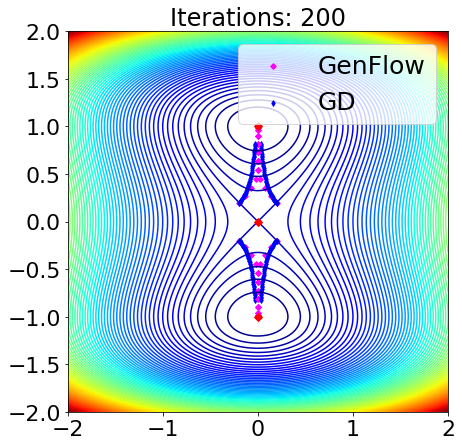}&
			\includegraphics[width=0.6\columnwidth]{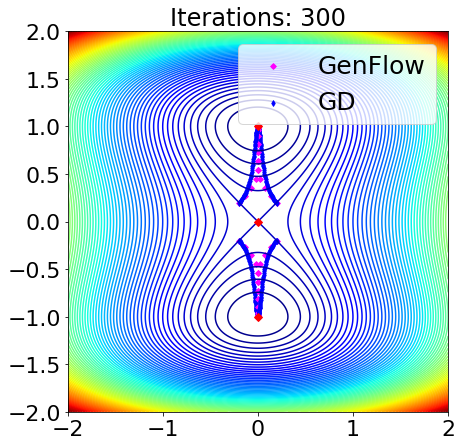}\cr
			(d)&(e)&(f)
		\end{tabular}
		\caption{\co{Snapshots at different iterations for evasion of saddle points. GenFlow(M) is able to evade the saddle points and converge to optimum in less than 50 iterations (slightly better than GenFlow), while gradient descent requires more than 300 iterations for convergence.}}
		\label{fig:SnapShots}
	\end{center}
	\vspace{-2em}
\end{figure*}

Having analyzed the theoretically superior convergence behavior of GenFlow and its momentum variant, we now focus our attention to empirical performance of the GenFlow scheme for a variety of tasks, involving (i) fast evasion of non-degenerate saddle points, (ii) training of deep neural networks for classification tasks, (iii) training of generative adversarial networks (GANs)~\cite{goodfellow2014generative}. The GenFlow and its momentum variant were implemented using PyTorch's inbuilt optimizer class, and can thus be easily integrated with any neural network training task. While GenFlow(M) is discretized as discussed in Remark~\ref{rem:discrete}, we use a simple forward-Euler discretization for GenFlow. Using similar techniques in \cite[Theorem 2]{garg2022}, it can be shown that under assumptions of strong-convexity and PL-inequality, the forward-Euler discretization results in $(T,\epsilon)$-close discrete-time approximations of GenFlow and GenFlow(M), however, its detailed analysis is beyond the scope of the current study and left as part of the future work. \co{In line with standard practice for optimization algorithms employing gradient normalization, we include a small numerical constant in the denominator term to prevent possible numerical issues related to division by zero. In the numerical implementation of GenFlow and GenFlow(M), we follow a similar approach by incorporating a small positive constant $10^{-7}$ to ensure numerical stability.}

\begin{figure*}[!ht]
	\begin{center}
		\begin{tabular}{ccc}
			\includegraphics[width=0.64\columnwidth]{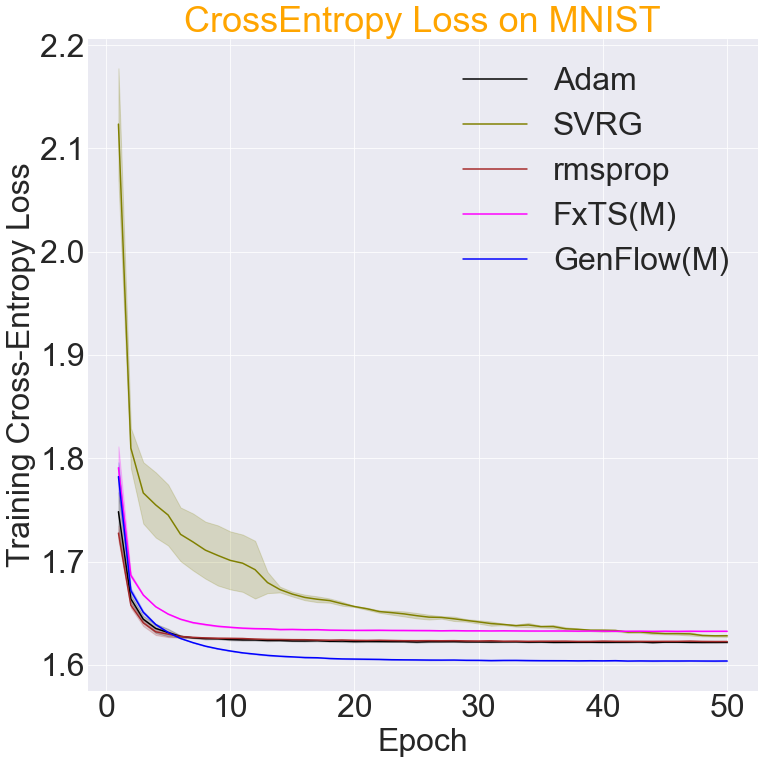}&
			\includegraphics[width=0.64\columnwidth]{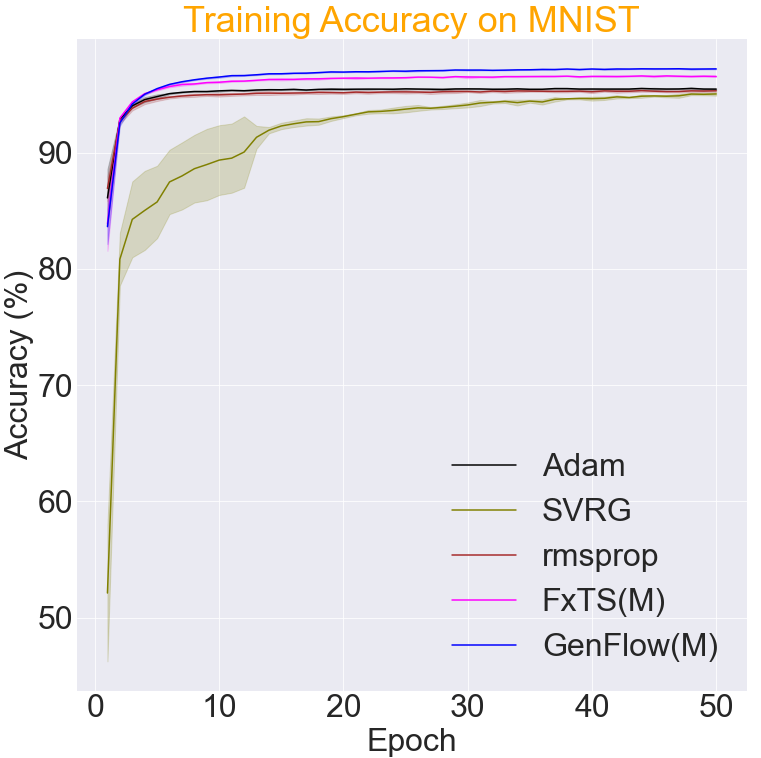}&
			\includegraphics[width=0.64\columnwidth]{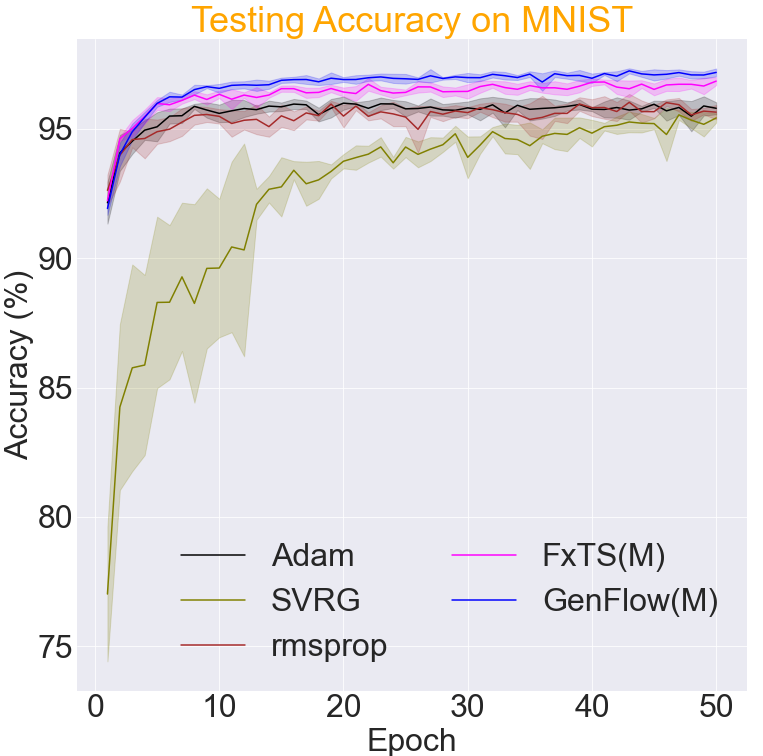}\cr
			(a)&(b)&(c)\cr
			\includegraphics[width=0.64\columnwidth]{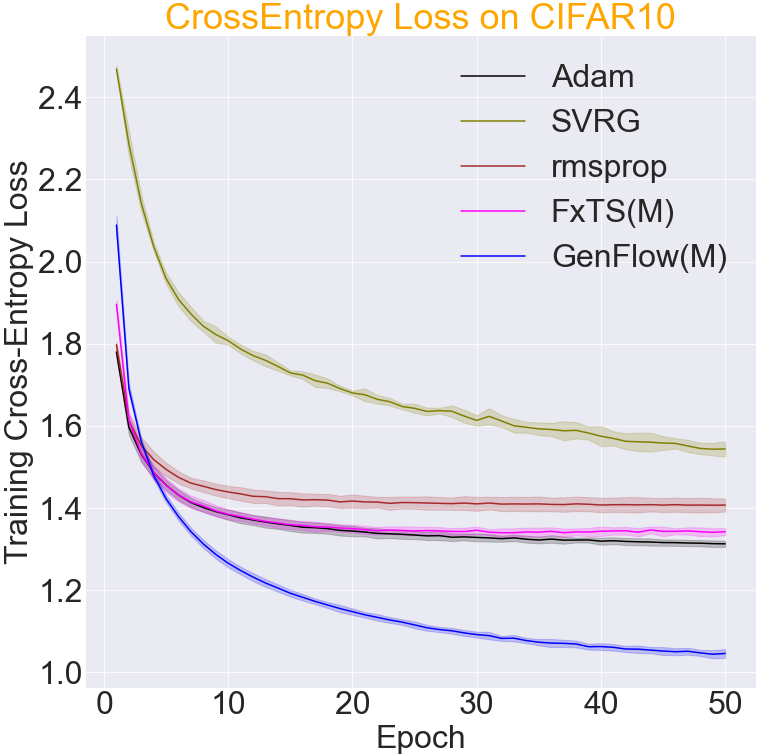}&
			\includegraphics[width=0.64\columnwidth]{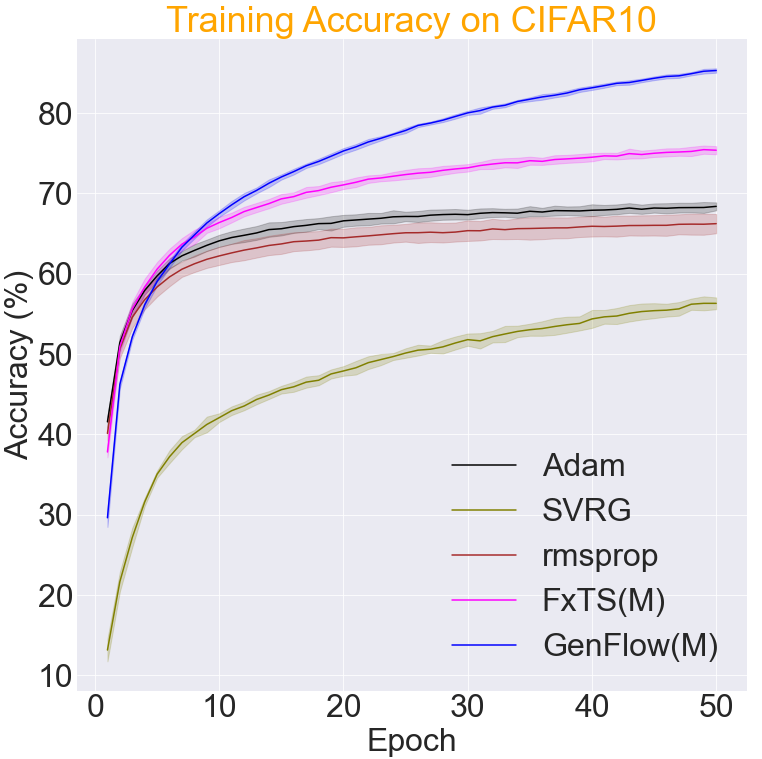}&
			\includegraphics[width=0.64\columnwidth]{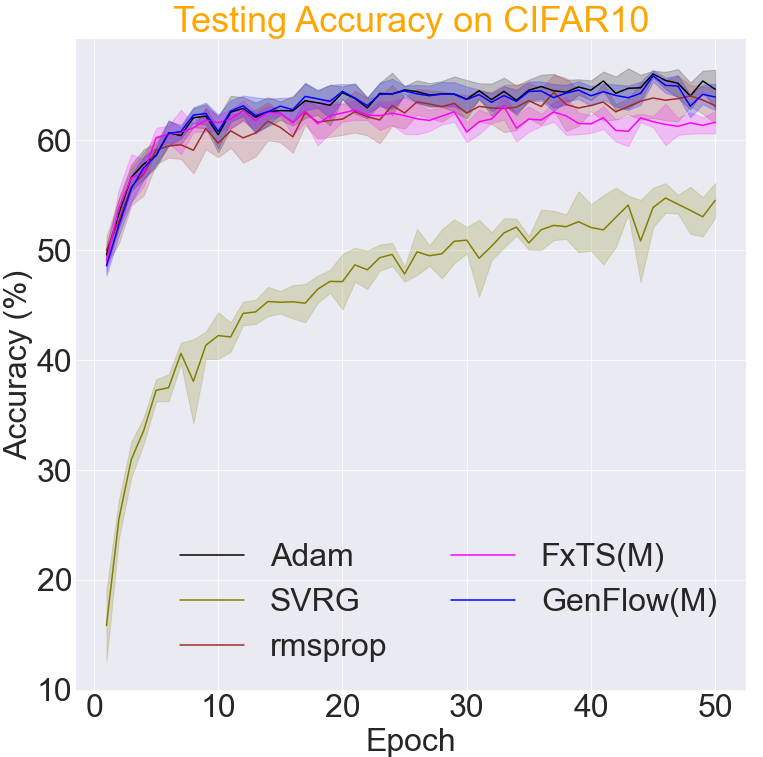}\cr
			(d)&(e)&(f)\cr
			\includegraphics[width=0.64\columnwidth]{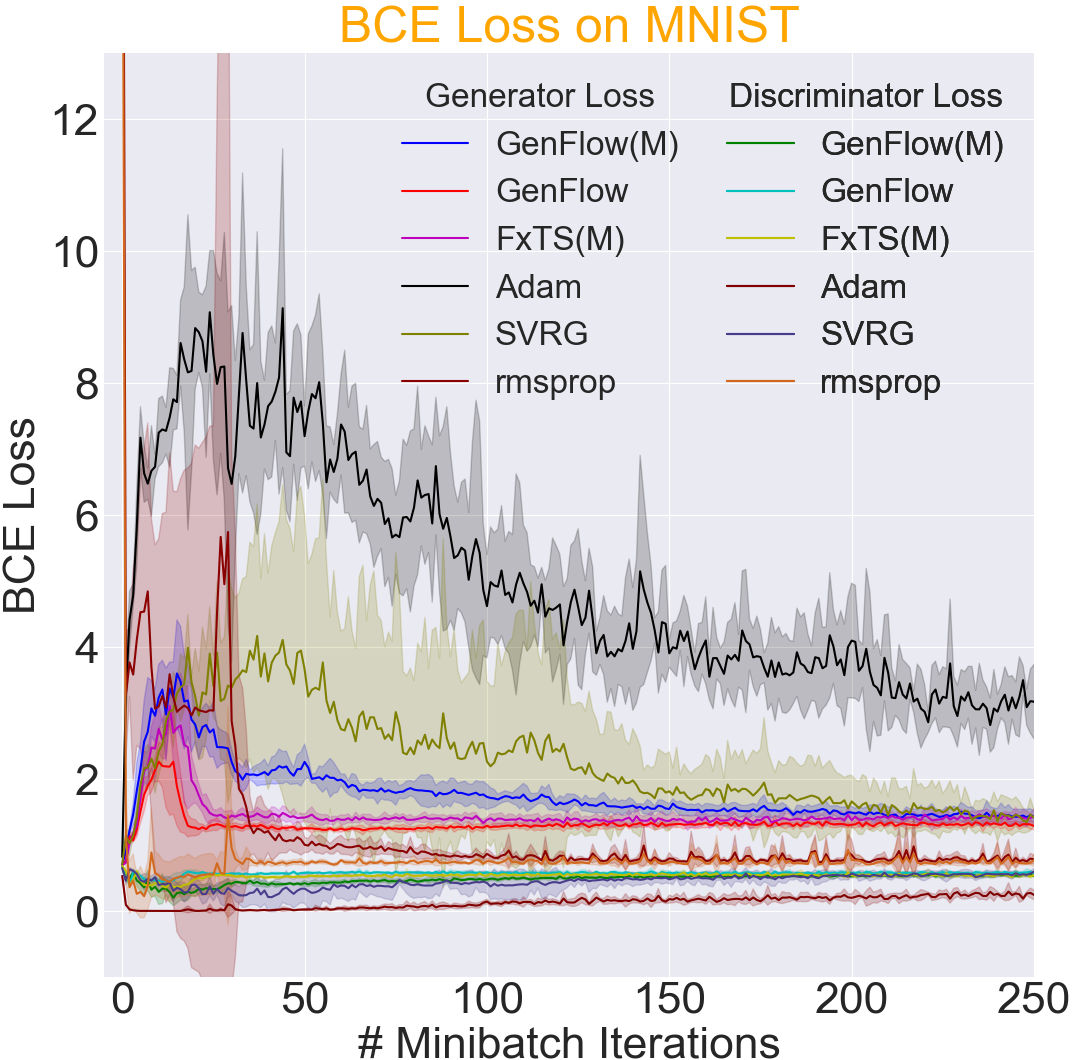}&
			\includegraphics[width=0.64\columnwidth]{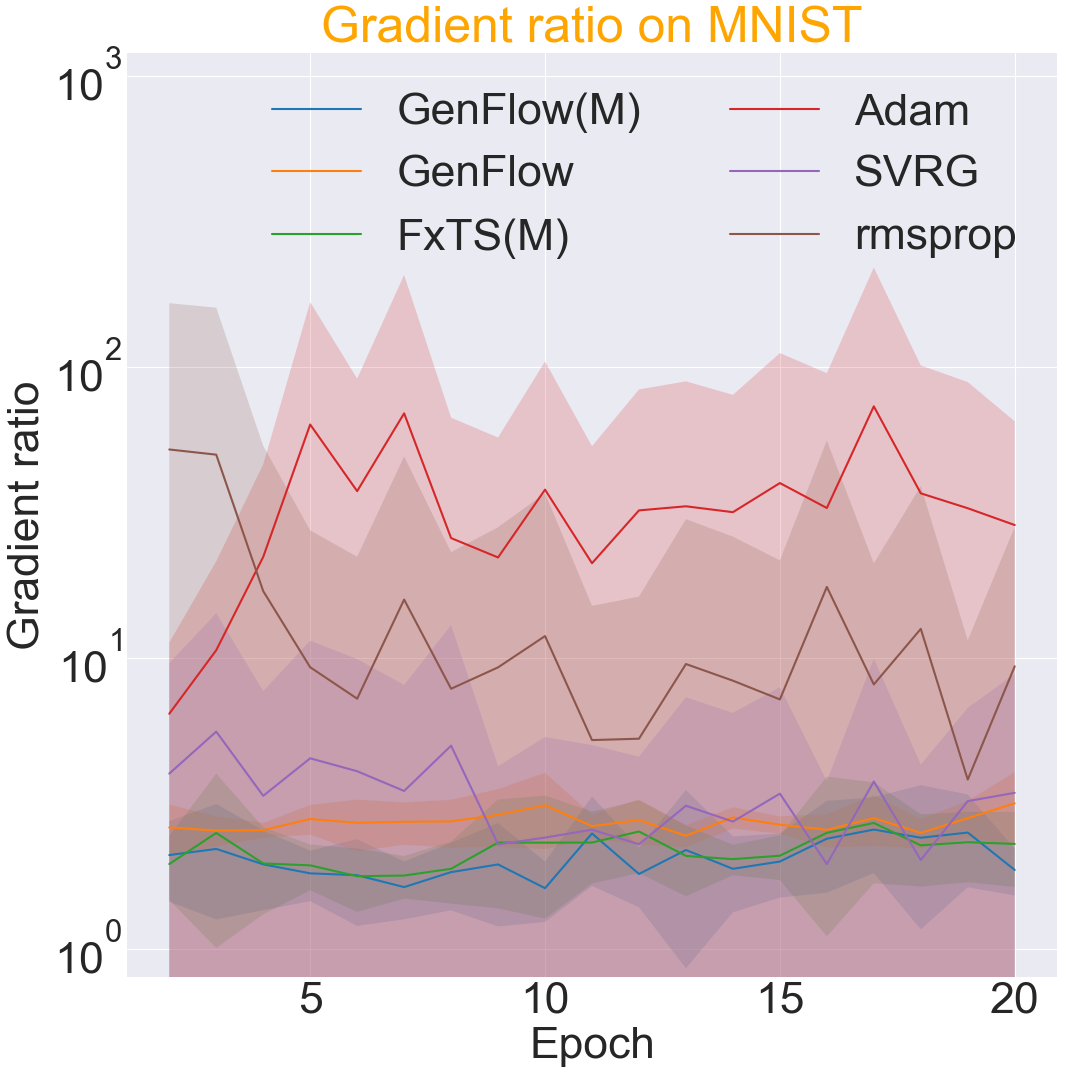}&
			\includegraphics[width=0.64\columnwidth]{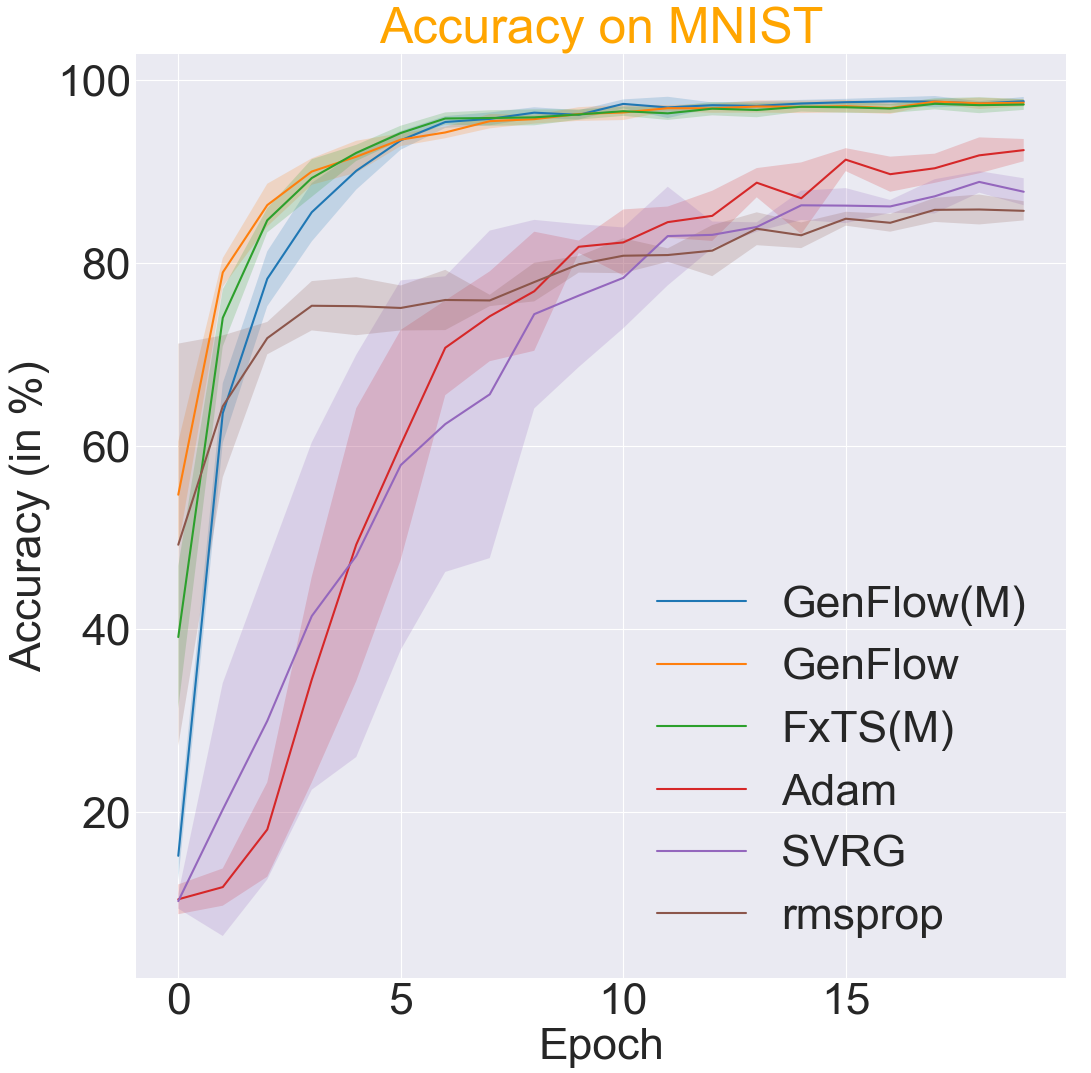}\cr
			(g)&(h)&(i)
		\end{tabular}
		\caption{Benchmarking GenFlow against SOTA optimizers for neural network based classification tasks and training stabilization of GANs: (a), (b), (c) indicate training and generalization performance of several optimizers on the MNIST dataset, while (d), (e), (f) analyze them on the CIFAR10 dataset. Items (g), (h), (i) depict training losses for GANs, gradient ratios for analyzing stabilization, and classification accuracy on images generated after each epoch, respectively.}
		\label{fig:Exp}
	\end{center}
\end{figure*}

\subsection{Fast Evasion of Non-Degenerate Saddle Points}

To illustrate fast evasion of saddle points, we consider the test function, $f(x,y)=0.5x^2+0.25y^4-0.5y^2$, described in~\cite{lee2016gd}. \co{Note that the function is strongly convex-strongly concave in its arguments.} The function admits three isolated critical points: $z_1=(0,0)$, $z_2=(0,-1)$, $z_3=(0,1)$, with $z_1$ being a non-degenerate saddle point, while $z_2$ and $z_3$ are isolated local minima. It is well known that the vanilla gradient descent may take exponential time to evade non-degenerate saddle points~\cite{du2017exp}. We, thus, benchmark GenFlow and \co{GenFlow(M)} against the gradient descent for saddle point evasion of the aforementioned test function. It can be seen in \com{Fig.~\ref{fig:SnapShots}} that while \co{all the algorithms} successfully escape the saddle point $z_1$, the time of evasion is much smaller for GenFlow \co{and GenFlow(M)}, as opposed to the gradient descent. 

\co{Interestingly, while the theoretical results on fixed-time saddle point evasion are specifically presented for GenFlow, similar arguments can be extended to analyze the saddle point evasion characteristics of GenFlow(M) as well. Notably,} \co{GenFlow(M) appears to be the fastest among the three methods due to the combined influence of element-wise normalization and accumulation of momentum.}

%%%

\subsection{Training of Deep Neural Networks}

The accelerated convergence behavior of GenFlow for training of neural network classifiers is benchmarked on the two most widely datasets: MNIST (60k training images, 10k test images)~\cite{lecun1998gradient}; and CIFAR10 (50k training images, 10k test images)~\cite{krizhevsky2009learning}. 

%Additionally, training stabilization of GANs is also benchmarked on the MNIST dataset.

\subsubsection{Baseline optimizers} We benchmark GenFlow(M) against the most widely used optimizers belonging to different families of first-order optimization algorithms: (i) RMSprop~\cite{hinton2012neural}, (ii) Stochastic Variance Reduced Gradient (SVRG)~\cite{reddi2016stochastic}, (iii) Adam~\cite{kingma2014adam}, and (iv) FxTS-GF(M)~\cite{budhraja2021breaking}.\footnote{\co{During our experiments, we observed that the performance of GenFlow is only marginally inferior to that of GenFlow(M). Therefore, in the interest of clearer visualization, we have not included the results from GenFlow for this case study.}} The hyperparameters for each algorithm are tuned for optimal performance. For instance, it was observed that a large learning rate for Adam or RMSprop significantly destabilizes the training, and hence a suitably tuned smaller learning rate was chosen for classification problems. In particular, we considered the following parameter choices to find the best performing hyperparameters: learning rates \com{(Euler discretization step-sizes)} in $\{10^{-4}, 2\times10^{-4}, 10^{-3}, 5\times10^{-3}, 10^{-2}, 5\times10^{-2}\}$ and momentum parameter \com{($\beta$)} in $\{0.1, 0.2, \dots, 0.9\}$. Similarly, the exponents $(p,q)$ utilized in gradient normalization of GenFlow and GenFlow(M), have well-defined ranges with the upper bound on convergence time decreasing as $p\to +\infty$ and $q\to 0^{+}$. In order to avoid any numerical instabilities upon discretization, these parameters are tuned gradually towards extreme values starting from $(p,q)=(2.1,1.99)$.

\subsubsection{Evaluation details} For the classification tasks, the optimization algorithms are evaluated on three criteria: (i) Training cross-entropy loss, (ii) Training accuracy, and (iii) Testing accuracy (for evaluating the quality of local minima in terms of generalization capability). 

%For the task of GAN training, the optimization algorithms are evaluated on three criteria: (i) Accuracy of images generated by the Generator (calculated using a pre-trained model available at: \url{github.com/aaron-xichen/pytorch-playground}), (ii) Binary Cross Entropy (BCE) Loss of the Generator and the Discriminator, and (iii) Gradient Ratio, the ratio is equal to $\|grad^{(t)}_{D}\|_{2}/\|grad^{(0)}_{D}\|_{2}+\|grad^{(t)}_{G}\|_{2}/\|grad^{(0)}_{G}\|_{2}$, where $grad^{(0)}_D$ (respectively $grad^{(0)}_G$) and $grad^{(t)}_D$ (respectively $grad^{(t)}_G$) are the initial and current gradients of Discriminator (respectively Generator)~\cite{jelassi2022adam}.

\subsubsection{Network details} For classification on the MNIST dataset, we have considered a network with a single convolutional layer with ReLU activation (consisting of 32 filters of size $3\times 3$), followed by a dense layer (with ReLU activation) of output size 128. The final linear layer transforms 128-dimensional input to a 10-dimensional output (corresponding to 10 classes) with SoftMax activation. Our architecture for CIFAR10 comprises of two convolutional layers, each followed by a max-pooling layer (with a $2\times2$ window). The convolutional layers comprise of 6 and 16 filters, each of size $5\times 5$ with ReLU activation, respectively. This is followed by two dense layers with ReLU activation (of sizes 120 and 84, respectively) and an output layer with SoftMax activation of size 10 (corresponding to 10 classes). We use $\ell_2$-regularized cross-entropy loss for the classification tasks. 

%For experiments on training GANs, Conditional GAN (cGAN) with the following architecture is used: (i) Generator network consists of an input layer, two hidden layers and an output layer. Both the input and hidden layers comprise of 2D transposed batch normalized convolutional layer with ReLU activation, while the output layer is a simple 2D transposed convolutional layer with \texttt{tanh} activation. The convolutional filter parameters, represented by the tuple (input dimension, output dimension, kernel size, stride), for the four layers in order are (74,256,3,2), (256,128,4,1), (128,64,3,2), and (64,1,4,2). (ii) Discriminator network is a simpler architecture comprising of input, hidden and output layers. Here, we use 2D transposed batch normalized convolutional layer with leaky ReLU activation (with 0.2 negative slope) for the input and the hidden layers. The output layer is a 2D convolutional layer without any nonlinear activation. The convolutional filter parameters for the three layers in order are (11,64,4,2), (64,128,4,2), and (128,1,4,2).

\subsubsection{Results}
For classification tasks on both MNIST and CIFAR10 datasets, we tune the hyperparameters for optimal performance. In particular, we arrived at the following optimizer settings for classification task on MNIST: (i) Adam (lr=$10^{-3}$), (ii) SVRG (lr=$10^{-2}$), (iii) RMSprop (lr=$10^{-3}$), (iv) FxTS(M) (lr=$10^{-2}$,momentum=0.7,$(p,q)=20,1.98$), (v) GenFlow(M) (lr=$10^{-2}$,momentum=0.9,$(p,q)=10,1.98$). \com{Fig.~\ref{fig:Exp}a} represents the training cross-entropy loss as a function of epochs, while \com{Figs.~\ref{fig:Exp}b} and \ref{fig:Exp}c depict the corresponding training and test accuracy obtained using the aforementioned optimizers, and averaged across five random seeds. It can be seen that GenFlow(M) not only arrives at an optimal solution faster than other methods, the variance across the random seeds is significantly small. \co{We attribute the lower variance of GenFlow(M) to its improved robustness against additive perturbations, a characteristic that can be demonstrated through a theoretical analysis similar to \cite{budhraja2021breaking}.} Both Adam and RMSprop could not be trained efficiently with higher learning rates, and the learning rates had to be reduced to $10^{-3}$ for optimal performance.

A similar behavior is observed on the CIFAR10 dataset (\com{Figs.~\ref{fig:Exp}d-f}). While the generalization performance of these optimizers on test dataset is largely very similar (barring the SVRG method), the performance on the training dataset in terms of $\ell_2$-regularized cross-entropy loss, as well as training accuracy, is far better for the proposed GenFlow(M). The optimal performances were obtained using following hyperparameters: (i) Adam (lr=$10^{-3}$), (ii) SVRG (lr=$10^{-3}$), (iii) RMSprop (lr=$10^{-3}$), (iv) FxTS(M) (lr=$5\times10^{-3}$,momentum=0.7,$(p,q)=20,1.98$), (v) GenFlow(M) (lr=$5\times10^{-3}$,momentum=0.5,$(p,q)=10,1.98$). Here, lr represents the learning rate of an optimization algorithm.

\subsection{Training of GANs}
GenFlow(M) and GenFlow are further evaluated by training GANs, modeled as minimax problems, on MNIST dataset. For this task, the optimization algorithms are evaluated on three criteria: (i) Accuracy of images generated by the Generator (calculated using a pre-trained model available at: \url{github.com/aaron-xichen/pytorch-playground}), (ii) Binary Cross Entropy (BCE) Loss of the Generator and the Discriminator, and (iii) Gradient Ratio, the ratio is equal to $\|grad^{(t)}_{D}\|_{2}/\|grad^{(0)}_{D}\|_{2}+\|grad^{(t)}_{G}\|_{2}/\|grad^{(0)}_{G}\|_{2}$, where $grad^{(0)}_D$ (respectively $grad^{(0)}_G$) and $grad^{(t)}_D$ (respectively $grad^{(t)}_G$) are the initial and current gradients of Discriminator (respectively Generator)~\cite{jelassi2022adam}.

\subsubsection{Network details}
Conditional GAN (cGAN) with the following architecture is used: (i) Generator network consists of an input layer, two hidden layers and an output layer. Both the input and hidden layers comprise of 2D transposed batch normalized convolutional layer with ReLU activation, while the output layer is a simple 2D transposed convolutional layer with \texttt{tanh} activation. The convolutional filter parameters, represented by the tuple (input dimension, output dimension, kernel size, stride), for the four layers in order are (74,256,3,2), (256,128,4,1), (128,64,3,2), and (64,1,4,2). (ii) Discriminator network is a simpler architecture comprising of input, hidden and output layers. Here, we use 2D transposed batch normalized convolutional layer with leaky ReLU activation (with 0.2 negative slope) for the input and the hidden layers. The output layer is a 2D convolutional layer without any nonlinear activation. The convolutional filter parameters for the three layers in order are (11,64,4,2), (64,128,4,2), and (128,1,4,2).

\subsubsection{Results}
The performance of GenFlow(M) and GenFlow is benchmarked against the FxTS-GF(M), Adam, SVRG and RMSprop optimizers. For each optimizer, best hyperparameters are selected based on Accuracy of images generated by the Generator for 20 epochs, averaged over 3 seeds. After finding the best hyperparameters, the plots in \com{Fig.~\ref{fig:Exp}g-i} are obtained for 5 seeds. The values of best hyperparameter tuple (generator step-size (lr), discriminator to generator step-size ratio, momentum (if applicable)) for GenFlow, GenFlow(M), FxTS-GF(M), Adam, SVRG and RMSprop are (0.01, 5, $p$ = 2.1, $q$ = 1.9, $\beta_1$ = 0.9, $\beta_2$ = 1.8), (0.005, 5, 0.3, $p$ = 2.1, $q$ = 1.9), (0.01, 5, 0.2, $p$ = 2.1, $q$ = 1.9), (0.0002, 10), (0.05, 5), and (0.005, 5, 0.5) respectively. From the \com{Fig.~\ref{fig:Exp}i}, we can see that GenFlow(M) and GenFlow outperforms other methods by achieving optimal solution faster and having low variance across seeds. \com{Fig.~\ref{fig:Exp}h} shows that GenFlow(M) and GenFlow has better training stability compared to other methods as the gradient ratio approximately stays near 1~\cite{jelassi2022adam}. Other details: (i) Mini-batch size = 128, (ii) Generator Gaussian noise dimension = 64 (iii) GenFlow hyperparameters $\beta_1$ and $\beta_2$ are multiplied with first and second terms in the RHS of \eqref{eq: fadagrad} respectively. Images generated by the generator of each of the optimizers after 20 epochs are shown in \com{Fig.~\ref{fig:GenImages}}. From the figure, it can be seen that the images generated by GenFlow(M) and GenFlow are qualitatively superior than the ones generated by Adam, RMSprop and SVRG optimizers. The difference is clearly visible when we compare digits, such as 4, 5, 6, and 7.

\begin{figure*}[!ht]
	\begin{center}
		\begin{tabular}{ccc}
			\includegraphics[width=0.6\columnwidth]{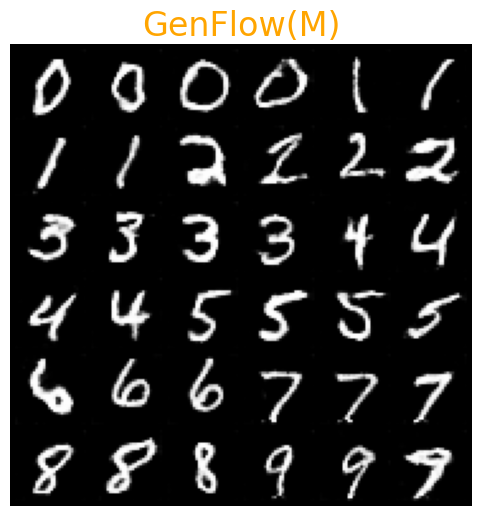}&
			\includegraphics[width=0.6\columnwidth]{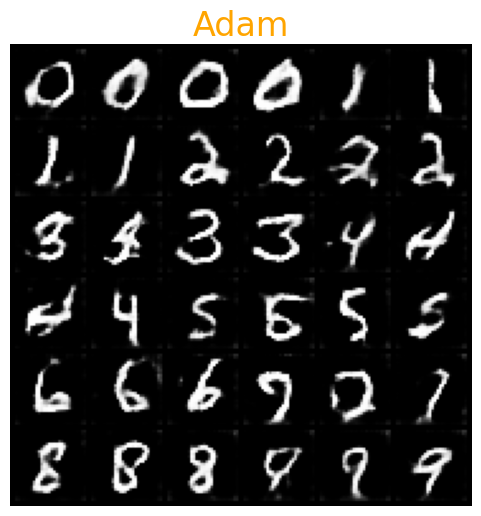}&
			\includegraphics[width=0.6\columnwidth]{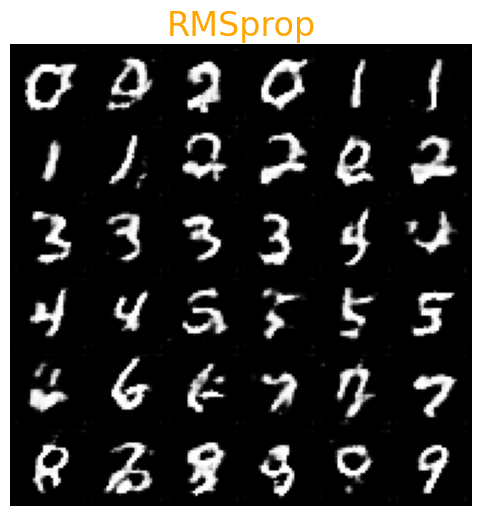}\cr
			(a)&(b)&(c)\cr
			\includegraphics[width=0.6\columnwidth]{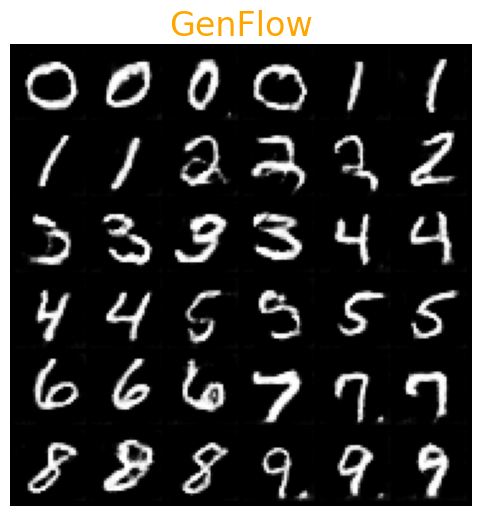}&
			\includegraphics[width=0.6\columnwidth]{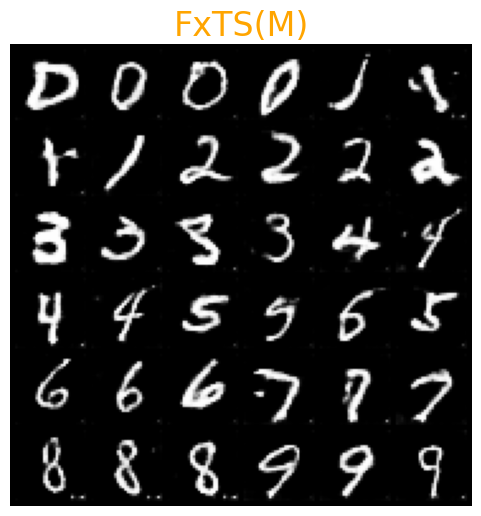}&
			\includegraphics[width=0.6\columnwidth]{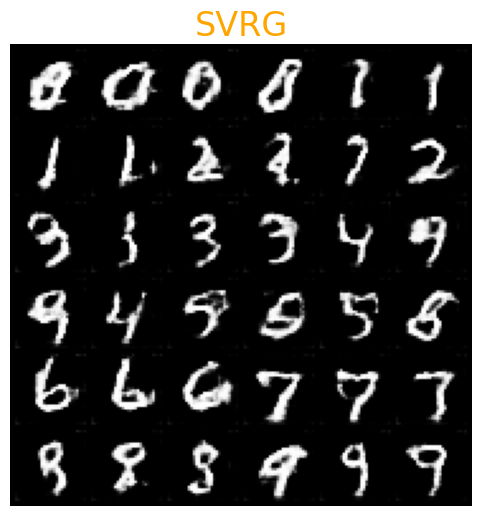}\cr
			(d)&(e)&(f)\cr
		\end{tabular}
		\caption{Comparison of generator generated images after 20 epochs trained using different optimizers.}
		\label{fig:GenImages}
	\end{center}
\end{figure*}

\section{Conclusion and Future Work}
In this work, we propose the {\it GenFlow} gradient-flow scheme and its momentum variant which were shown to converge to the optimum of cost functions satisfying PL-inequality in fixed time (i.e., the convergence time is finite and uniformly upper bounded irrespective of initial candidate solution). We further bound escape time from non-degenerate saddle points and show fixed-time convergence to the solution of minimax problems under suitable assumptions. Experimental results show superior convergence properties compared to state-of-the-art optimizers in training neural networks and GANs.

\co{Future research directions involve analyzing consistent discretization strategies of the proposed GenFlow and GenFlow(M). In particular, we aim to explore the conditions under which the discretized GenFlow algorithm converges to an $\epsilon$-neighborhood around the optimal solution within a uniformly bounded number of iterations. Of particular interest is the provision of upper bounds on the discretization step-size, considering additional assumptions of Lipschitz smoothness. Exploring faster momentum-based methods for solving saddle point problems is an area that has received relatively less attention. In our future work, we aim to extend the proposed saddle-point dynamics to incorporate a momentum variant, which has the potential to further accelerate convergence. Finally, we also aim to analyze the variance reduction properties and robustness of the proposed algorithms
to noisy gradients in a follow up work.}

%%%%%%%%%%%%%%%%%%%%%%%%%%%%%%%%%%%%%%%%%%%%%%%
%%%%%%%%%%%%%%%%%%%%%%%%%%%%%%%%%%%%%%%%%%%%%%%
%%%%%%%%%%%%%%%%%%%%%%%%%%%%%%%%%%%%%%%%%%%%%%%

\appendices
\section{Proof of \lemref{lemma: norm_convert}}\label{append:A}
\begin{proof}
    The matrix $H$ is a symmetric matrix. Thus its eigenvalue decomposition is given by $H=S D S^\top$ where $S$ is an orthogonal matrix and $D$ is a diagonal matrix of the eigenvalues of $H$. Clearly $|D|=\sqrt{D^\top D}$ is a diagonal matrix with modulus of the eigenvalues $\lambda \in \sigma(H)$ at the diagonal. Observe that $H^\top H = SD^2 S^\top$ and it can be easily proved that $|H| = \sqrt{H^\top H} = S|D|S^\top$. Now
    \begin{align*}
    \Tilde{d}(x) &= \sqrt{(x-x^*)^\top |H|(x-x^*)} \\
    &= \sqrt{(x-x^*)^\top S|D|S^T (x-x^*)} \\
    &\leq \sqrt{(x-x^*)^\top S \Lambda_{\max}I_n S^\top (x-x^*)}\\
    &= \sqrt{(x-x^*)^\top \Lambda_{\max} (x-x^*)} \\
    &= \sqrt{\Lambda_{\max}}\|x-x^*\|, \numberthis \label{eq: d_lemma_ubound}
    \end{align*}
    where to prove the inequality the following fact has been used $\Lambda_{\max}I_n - |D| \succeq 0$. Thus, when $\|x-x^*\|\leq \frac{a}{\sqrt{\Lambda_{\max}}}$, then \eqref{eq: x_to_dx} holds. Similarly, we obtain
    \begin{align*}
    \Tilde{d}(x) &\geq \sqrt{(x-x^*)^\top S \Lambda_{\min}I_n S^\top (x-x^*)} \\
    &= \sqrt{\Lambda_{\min}}\|x-x^*\|, \numberthis \label{eq: d_lemma_lbound}
    \end{align*}
    where the inequality is proved using $|D|-\Lambda_{\min}I_n \succeq 0$. If $\Tilde{d}(x)\leq a$ then by \eqref{eq: d_lemma_lbound} we obtain $\sqrt{\Lambda_{\min}}\|x - x^*\| \leq a$, i.e., \eqref{eq: dx_to_x} is established. This concludes the proof.
\end{proof}

\section{Proof of \lemref{lemma: taylor_approx}}\label{append:B}
\begin{proof}
    For the function $f$, we apply Taylor's theorem at $x^*$ to obtain
    \begin{align*}
    f(x) &= f(x^*) + \nabla f(x^*)^\top (x-x^*) \\
        &\ \ + \frac{1}{2} (x-x^*)^\top D^2 (f(x^*)) (x-x^*) + R_2 (x),
    \end{align*}
    where $R_2 (x)\rightarrow 0$ as $\|x-x^*\|\rightarrow 0$. Using the fact $\nabla f(x^*) =0$, we obtain
    \begin{align*}
    |f(x) - f(x^*)| 
    &= \left| \frac{1}{2} (x-x^*)^\top D^2 (f(x^*)) (x-x^*) + R_2 (x) \right| \\
    &\leq \frac{1}{2}\left| (x-x^*)^\top D^2 (f(x^*)) (x-x^*) \right|\!+\!|R_2 (x)| \\
    &\leq \frac{1}{2}\left| (x-x^*)^\top |H| (x-x^*) \right|\!+\!|R_2 (x)| \\
    &= \frac{1}{2} \Tilde{d}(x)^2 + |R_2 (x)|,
    \end{align*}
    where to obtain the second inequality, the relation $|H|-H \succeq 0$ is used. Since $\lim_{\|x-x^*\|\to 0} R_2(x)=0$ for some neighbourhood of $x^*$, for some $k_1 > 0.5$ it is true that $k_1 \Tilde{d}(x)^2 \geq \frac{1}{2}\Tilde{d}(x)^2 + |R_2 (x)|$. Thus, the inequality \eqref{eq: bound_f} is proved.
    
    We now \co{invoke} Taylor's theorem for $\nabla f(x)$ at $x^*$ to obtain 
    \begin{align*}
    \nabla f(x) &= \nabla f(x^*) + D^2 (f(x^*)) (x-x^*) + R_1 (x) \\
    &= H(x-x^*) + R_1 (x),
    \end{align*}
    where the remainder term $R_1 (x)$ satisfies $\lim_{\|x-x^*\|_1 \to 0} \|R_1(x)\|_1 =0$.
    % Reference needed
    Reverse triangle inequality with respect to $\ell_1$ norm yields
    \begin{align*}
    \|\nabla f(x)\|_1 &= \|H(x-x^*) + R_1 (x)\|_1 \\
    &\geq \|H(x-x^*)\|_1 - \|R_1 (x)\|_1 .
    \end{align*}
    Now, for some $k_2 <1$ it can be proved that $\|H(x-x^*)\|_1 - \|R_1 (x)\|_1 \geq k_2 \|H(x-x^*)\|_1$ holds for some neighbourhood of $x^*$ which leads to \eqref{eq: taylor_2}, and concludes the proof.
\end{proof}

\section{Proof of \thmref{th: fadagrad_evasion}}\label{append:C}
\begin{proof}
Without loss of generality, assume $x^*=0$ and let $H:=D^2(f(0))$. For $x\in \mathbb{R}^n$, define $\Tilde{d}(x):= \sqrt{x^\top|H|x}$, where $|B|:=\sqrt{B^\top B}$ for a square matrix $B$. Since the assumptions of Lemma \ref{lemma: norm_convert}-\ref{lemma: taylor_approx} are satisfied, their results can be used here. This also means that for some neighbourhood of $0\in \mathbb{R}^n$, the relations \eqref{eq: bound_f} and \eqref{eq: taylor_2} hold true. Now 
\begin{align*}
\|\nabla f(x(t))\|_1 &\geq k_2\|Hx(t)\|_1 \geq k_2\|Hx(t)\|_2 \\
&=k_2\||H|^{\frac{1}{2}}|H|^{\frac{1}{2}}x(t)\|_2 \\ 
&\geq k_2\sqrt{\Lambda_{\min}} \||H|^{\frac{1}{2}}x(t)\|_2 = k_2\sqrt{\Lambda_{\min}}\Tilde{d}(x(t)). \numberthis \label{eq: bound_gradient_norm}
\end{align*}
To derive the second inequality above, the fact $\|v\|_1 \geq \|v\|_2$ for $v\in\mathbb{R}^n$ is used.

Using the chain rule, we obtain 
\begin{align*}
\frac{d}{dt}f(x(t)) &= \sum_{i=1}^n \left( -|\nabla_i f(x(t))|^{\frac{p}{p-1}} -|\nabla_i f(x(t))|^{\frac{q}{q-1}} \right) 
\\
&= -\|\nabla f(x(t))\|_{\frac{p}{p-1}}^{\frac{p}{p-1}} -\|\nabla f(x(t))\|_{\frac{q}{q-1}}^{\frac{q}{q-1}} 
\\ &\leq -n^{\frac{-1}{p-1}}\|\nabla f(x(t))\|_{1}^{\frac{p}{p-1}}-n^{\frac{-1}{q-1}}\|\nabla f(x(t))\|_{1}^{\frac{q}{q-1}}, 
\end{align*}
where to derive the inequality, the relation $\|v\|_r \leq n^{\frac{1}{r}-\frac{1}{s}} \|v\|_s$ for $v\in \mathbb{R}^n$ and $s>r\geq1$ has been used.

We now use the inequality \eqref{eq: bound_gradient_norm} to obtain the following bound on the rate of change of $f(x(t))$:
\begin{align}\label{eq: rate_bound}
    -\frac{d}{dt}f(x(t)) &\geq \underbrace{\left(n^{\frac{-1}{p}} k_2 \sqrt{\Lambda_{\min}}\right)^{\frac{p}{p-1}}}_{k_3}\Tilde{d}(x)^{\frac{p}{p-1}} \nonumber\\
    &\ \ + \underbrace{\left(n^{\frac{-1}{q}} k_2 \sqrt{\Lambda_{\min}}\right)^{\frac{q}{q-1}}}_{k_4}\Tilde{d}(x)^{\frac{q}{q-1}} \nonumber \\
    \implies -\frac{d}{dt}f(x(t)) &\geq k_3\Tilde{d}(x)^{\frac{p}{p-1}} + k_4 \Tilde{d}(x)^{\frac{q}{q-1}}.
\end{align}

% Let $k_3=\left(n^{\frac{-1}{p}}k_2 \sqrt{\Lambda_{\min}}\right)^{\frac{p}{p-1}}$ and $k_4=\left(n^{\frac{-1}{q}} k_2 \sqrt{\Lambda_{\min}}\right)^{\frac{q}{q-1}}$. Then, one obtains
% \begin{equation} \label{eq: rate_bound}
% -\frac{d}{dt}f(x(t)) \geq k_3\Tilde{d}(x)^{\frac{p}{p-1}} + k_4 \Tilde{d}(x)^{\frac{q}{q-1}}.
% \end{equation}
Let the approximations \eqref{eq: bound_f} and \eqref{eq: bound_gradient_norm} be true inside the ball $B_{\hat{r}}(0)$ for $\hat{r}>0$. Suppose that $x(t)\in B_{\hat{r}}(0)$ for $t\in [t_1,t_2]$. Let $e(t):=\Tilde{d}(x(t))$ and now integrate \eqref{eq: rate_bound} to get 
\begin{equation*}
f(x(t_1))-f(x(t_2)) \geq k_3 \int_{t_1}^{t_2} e(s)^{\frac{p}{p-1}}ds + k_4 \int_{t_1}^{t_2} e(s)^{\frac{q}{q-1}}ds.
\end{equation*}
Let $r:= \kappa^{-\frac{1}{2}}\hat{r}$ where $\kappa=\frac{\Lambda_{\max}}{\Lambda_{\min}}$. For $\eta\leq \sqrt{\Lambda_{\max}}r$ it is true that $\Tilde{d}(x)\leq \eta$ implies $x\in B_{\hat{r}}(0)$. Let $t_0>0$ be the first time where $e(t)\leq \eta$ and $t_3$ be the last time when $e(t)\leq \eta$, i.e., $t_3 = \sup\{t\in[0,\infty] : e(t)\leq \eta\}$. The time $t_3$ is not assumed to be finite at this point but later it is proved that the length of the interval $[t_0, t_3]$ is bounded which in turn implies that $t_3$ is finite. If $t_3=\infty$, then in an abuse of notation let $f(x(\infty))=\lim_{t\to\infty}f(x(t))$, where note that the limit exists since $f(x(t))$ is monotonically non-increasing by \eqref{eq: rate_bound}. Integrate over the interval $[t_0, t_3]$ to get the following 
\begin{align*}
    f(x(t_0))-f(x(t_3)) &= \int_{t_0}^{t_3} -\frac{d}{ds}f(x(s))ds \\
    &\geq \int_{e(s)\leq\eta} -\frac{d}{ds}f(x(s))ds \\
    &\geq k_3 \int_{e(s)\leq\eta} e(s)^{\frac{p}{p-1}}ds \\
    &\ \ + k_4 \int_{e(s)\leq\eta} e(s)^{\frac{q}{q-1}}ds,
\end{align*}
where the fact $\frac{d}{dt}f(x(t))\leq 0$ has been used to derive the first inequality and the second inequality has been derived using the inequality \eqref{eq: rate_bound} over the intervals where $e(\cdot)\leq \eta$. Add and subtract $f(0)$ in the left hand side of the above inequality along with the bound \eqref{eq: bound_f} to get 
\begin{equation*}
2k_1\eta^2 \geq k_3 \int_{e(s)\leq\eta} e(s)^{\frac{p}{p-1}}ds + k_4 \int_{e(s)\leq\eta} e(s)^{\frac{q}{q-1}}ds.
\end{equation*}
As $e(\cdot)\geq 0$, the following two inequalities can be deduced
\begin{align}
\int_{e(s)\leq\eta} e(s)^{\frac{p}{p-1}}ds &\leq \frac{2k_1}{k_3}\eta^2, \label{eq: star}\\ 
\int_{e(s)\leq\eta} e(s)^{\frac{q}{q-1}}ds &\leq \frac{2k_1}{k_4}\eta^2. \label{eq: double_star}
\end{align}

Using Markov's inequality \cite{federer2014measure}, we obtain 
\begin{align*}
    \mathcal{L}^1&\left(\left\{s:  \frac{\eta^{\frac{p}{p-1}}}{2} \leq e(s)^{\frac{p}{p-1}} \leq \eta^{\frac{p}{p-1}} \right\}\right) \\
    &\leq \frac{2}{\eta^{\frac{p}{p-1}}}\int_{e(s)^{\frac{p}{p-1}}\leq \eta^{\frac{p}{p-1}}} e(s)^{\frac{p}{p-1}}ds \\
    &= \frac{2}{\eta^{\frac{p}{p-1}}}\int_{e(s)\leq \eta} e(s)^{\frac{p}{p-1}}ds \leq \frac{4k_1}{k_3} \eta^{\frac{p-2}{p-1}},
\end{align*}
where to obtain the final inequality the relation \eqref{eq: star} has been used. We now iteratively apply the Markov's inequality to obtain 
\begin{align*}
\mathcal{L}^1&\left(\{s:0< e(s)^{\frac{p}{p-1}}\leq \eta^{\frac{p}{p-1}}\}\right) \\
&= \sum_{i=0}^{\infty} \mathcal{L}^1 \left(\left\{s : \frac{\eta^{\frac{p}{p-1}}}{2^{i+1}} \leq e(s)^{\frac{p}{p-1}} \leq \frac{\eta^{\frac{p}{p-1}}}{2^{i}}\right\}\right) \\
&\leq \sum_{i=0}^{\infty} \frac{4k_1}{k_3} \frac{\eta^{\frac{p-2}{p-1}}}{2^i} = \frac{8k_1}{k_3} \eta^{\frac{p-2}{p-1}}.
\end{align*}
Since $ \{s:0<e(s)\leq \eta\}= \{s:0< e(s)^{\frac{p}{p-1}}\leq \eta^{\frac{p}{p-1}}\}$ one gets 
\begin{equation} \label{eq: d_tilde_bound}
\mathcal{L}^1\left(\{s:0<e(s)\leq \eta\}\right) \leq \frac{8k_1}{k_3} \eta^{\frac{p-2}{p-1}}.
\end{equation}

Observe that $\{s: 0<\|x(s)\|\leq r\} \subset \{s: 0<\Tilde{d}(x(s))\leq \sqrt{\Lambda_{\max}}r\}$. Substituting $\eta=\sqrt{\Lambda_{\max}}r$ in \eqref{eq: d_tilde_bound}, we obtain
\begin{equation} \label{eq: evasion_bound1}
\mathcal{L}^1\left(\{s: 0<\|x(s)\|\leq r\}\right) \leq n^{\frac{1}{p-1}}\frac{8k_1}{k_2^{\frac{p}{p-1}}}\frac{\Lambda_{\max}^{\frac{p-2}{2(p-1)}}}{\Lambda_{\min}^{\frac{p}{2(p-1)}}}r^{\frac{p-2}{p-1}}. 
\end{equation}
Instead of using \eqref{eq: star}, one can use \eqref{eq: double_star} and follow similar steps to obtain the following set of inequalities:
\begin{align*}
\mathcal{L}^1\left(\{s:0< e(s)^{\frac{q}{q-1}}\leq \eta^{\frac{q}{q-1}}\}\right) &\leq \frac{8k_1}{k_4} \eta^{\frac{q-2}{q-1}}, \\
\mathcal{L}^1\left(\{s:0<e(s)\leq \eta\}\right) &\leq \frac{8k_1}{k_4} \eta^{\frac{q-2}{q-1}}, \\
\mathcal{L}^1\left(\{s: 0<\|x(s)\|\leq r\}\right) &\leq n^{\frac{1}{q-1}}\frac{8k_1}{ k_2^{\frac{q}{q-1}}}\frac{\Lambda_{\max}^{\frac{q-2}{2(q-1)}}}{\Lambda_{\min}^{\frac{q}{2(q-1)}}}r^{\frac{q-2}{q-1}}. \numberthis \label{eq: evasion_bound2}
\end{align*}
Both \eqref{eq: evasion_bound1} and \eqref{eq: evasion_bound2} are upper bounds on the length of time spent in $B_r(x^*)\setminus\{x^*\}$ so it follows that both are satisfied if the smaller of the two bounds is satisfied. This concludes the proof. 
\end{proof}

\bibliographystyle{IEEEtran}
\bibliography{myRef}

\begin{thebibliography}{10}
\providecommand{\url}[1]{#1}
\csname url@samestyle\endcsname
\providecommand{\newblock}{\relax}
\providecommand{\bibinfo}[2]{#2}
\providecommand{\BIBentrySTDinterwordspacing}{\spaceskip=0pt\relax}
\providecommand{\BIBentryALTinterwordstretchfactor}{4}
\providecommand{\BIBentryALTinterwordspacing}{\spaceskip=\fontdimen2\font plus
\BIBentryALTinterwordstretchfactor\fontdimen3\font minus
  \fontdimen4\font\relax}
\providecommand{\BIBforeignlanguage}[2]{{%
\expandafter\ifx\csname l@#1\endcsname\relax
\typeout{** WARNING: IEEEtran.bst: No hyphenation pattern has been}%
\typeout{** loaded for the language `#1'. Using the pattern for}%
\typeout{** the default language instead.}%
\else
\language=\csname l@#1\endcsname
\fi
#2}}
\providecommand{\BIBdecl}{\relax}
\BIBdecl

\bibitem{bottou2018optimization}
L.~Bottou, F.~E. Curtis, and J.~Nocedal, ``Optimization methods for large-scale
  machine learning,'' \emph{SIAM Review}, vol.~60, no.~2, pp. 223--311, 2018.

\bibitem{beck2017first}
A.~Beck, \emph{First-order methods in optimization}.\hskip 1em plus 0.5em minus
  0.4em\relax SIAM, 2017.

\bibitem{polyak1964heavyball}
B.~Polyak, ``Some methods of speeding up the convergence of iteration
  methods,'' \emph{USSR Computational Mathematics and Mathematical Physics},
  vol.~4, no.~5, pp. 1--17, 1964.

\bibitem{nesterov1983nag}
Y.~Nesterov, ``A method for unconstrained convex minimization problem with the
  rate of convergence $o(1/k^2)$,'' vol. 269.\hskip 1em plus 0.5em minus
  0.4em\relax Doklady AN SSR, 1983, pp. 543--547.

\bibitem{li2020accelerated}
H.~Li, C.~Fang, and Z.~Lin, ``Accelerated first-order optimization algorithms
  for machine learning,'' \emph{Proceedings of the IEEE}, vol. 108, no.~11, pp.
  2067--2082, 2020.

\bibitem{wibisono2016variational}
A.~Wibisono, A.~C. Wilson, and M.~I. Jordan, ``A variational perspective on
  accelerated methods in optimization,'' \emph{Proceedings of the National
  Academy of Sciences}, vol. 113, no.~47, pp. E7351--E7358, 2016.

\bibitem{kovachki2021continuous}
N.~B. Kovachki and A.~M. Stuart, ``Continuous time analysis of momentum
  methods,'' \emph{Journal of Machine Learning Research}, vol.~22, no.~17, pp.
  1--40, 2021.

\bibitem{khalil2002nonlinear}
H.~Khalil, \emph{Nonlinear Systems}.\hskip 1em plus 0.5em minus 0.4em\relax
  Prentice Hall, 2002.

\bibitem{lessard2016analysis}
L.~Lessard, B.~Recht, and A.~Packard, ``Analysis and design of optimization
  algorithms via integral quadratic constraints,'' \emph{SIAM Journal on
  Optimization}, vol.~26, no.~1, pp. 57--95, 2016.

\bibitem{xu2018accelerated}
P.~Xu, T.~Wang, and Q.~Gu, ``Accelerated stochastic mirror descent: From
  continuous-time dynamics to discrete-time algorithms,'' in
  \emph{International Conference on Artificial Intelligence and
  Statistics}.\hskip 1em plus 0.5em minus 0.4em\relax PMLR, 2018, pp.
  1087--1096.

\bibitem{orvieto2019continuous}
A.~Orvieto and A.~Lucchi, ``Continuous-time models for stochastic optimization
  algorithms,'' \emph{Advances in Neural Information Processing Systems},
  vol.~32, 2019.

\bibitem{romero2020finite}
O.~Romero and M.~Benosman, ``Finite-time convergence in continuous-time
  optimization,'' in \emph{International Conference on Machine Learning}.\hskip
  1em plus 0.5em minus 0.4em\relax PMLR, 2020, pp. 8200--8209.

\bibitem{budhraja2021breaking}
P.~Budhraja, M.~Baranwal, K.~Garg, and A.~Hota, ``Breaking the convergence
  barrier: Optimization via fixed-time convergent flows,'' in \emph{Proceedings
  of the AAAI Conference on Artificial Intelligence}, vol.~36, no.~6, 2022, pp.
  6115--6122.

\bibitem{bhat2000}
S.~P. Bhat and D.~S. Bernstein, ``Finite-time stability of continuous
  autonomous systems,'' \emph{SIAM Journal on Control and Optimization},
  vol.~38, no.~3, pp. 751--766, 2000.

\bibitem{cortes2006}
J.~Cortés, ``Finite-time convergent gradient flows with applications to
  network consensus,'' \emph{Automatica}, vol.~42, no.~11, pp. 1993--2000,
  2006.

\bibitem{polyakov2011nonlinear}
A.~Polyakov, ``Nonlinear feedback design for fixed-time stabilization of linear
  control systems,'' \emph{IEEE Transactions on Automatic Control}, vol.~57,
  no.~8, pp. 2106--2110, 2011.

\bibitem{garg2021}
K.~Garg and D.~Panagou, ``Fixed-time stable gradient flows: Applications to
  continuous-time optimization,'' \emph{IEEE Transactions on Automatic
  Control}, vol.~66, no.~5, pp. 2002--2015, 2021.

\bibitem{karimi2016linear}
H.~Karimi, J.~Nutini, and M.~Schmidt, ``Linear convergence of gradient and
  proximal-gradient methods under the polyak-{\l}ojasiewicz condition,'' in
  \emph{Joint European Conference on Machine Learning and Knowledge Discovery
  in Databases}.\hskip 1em plus 0.5em minus 0.4em\relax Springer, 2016, pp.
  795--811.

\bibitem{frei2021proxy}
S.~Frei and Q.~Gu, ``Proxy convexity: A unified framework for the analysis of
  neural networks trained by gradient descent,'' \emph{Advances in Neural
  Information Processing Systems}, vol.~34, pp. 7937--7949, 2021.

\bibitem{liu2022loss}
C.~Liu, L.~Zhu, and M.~Belkin, ``Loss landscapes and optimization in
  over-parameterized non-linear systems and neural networks,'' \emph{Applied
  and Computational Harmonic Analysis}, vol.~59, pp. 85--116, 2022.

\bibitem{du2018gradient}
\BIBentryALTinterwordspacing
S.~S. Du, X.~Zhai, B.~Poczos, and A.~Singh, ``Gradient descent provably
  optimizes over-parameterized neural networks,'' in \emph{International
  Conference on Learning Representations}, 2019. [Online]. Available:
  \url{https://openreview.net/forum?id=S1eK3i09YQ}
\BIBentrySTDinterwordspacing

\bibitem{allen2019convergence}
Z.~Allen-Zhu, Y.~Li, and Z.~Song, ``A convergence theory for deep learning via
  over-parameterization,'' in \emph{International Conference on Machine
  Learning}.\hskip 1em plus 0.5em minus 0.4em\relax PMLR, 2019, pp. 242--252.

\bibitem{levy2016ngd}
K.~Y. Levy, ``The power of normalization: Faster evasion of saddle points,''
  \emph{arXiv preprint arXiv:1611.04831}, 2016.

\bibitem{murray2019ngd}
R.~Murray, B.~Swenson, and S.~Kar, ``Revisiting normalized gradient descent:
  Fast evasion of saddle points,'' \emph{IEEE Transactions on Automatic
  Control}, vol.~64, no.~11, pp. 4818--4824, 2019.

\bibitem{wilson2019accelerating}
A.~C. Wilson, L.~Mackey, and A.~Wibisono, ``Accelerating rescaled gradient
  descent: Fast optimization of smooth functions,'' \emph{Advances in Neural
  Information Processing Systems}, vol.~32, 2019.

\bibitem{bottou2010large}
L.~Bottou, ``Large-scale machine learning with stochastic gradient descent,''
  in \emph{Proceedings of COMPSTAT'2010}.\hskip 1em plus 0.5em minus
  0.4em\relax Springer, 2010, pp. 177--186.

\bibitem{kingma2014adam}
D.~P. Kingma and J.~Ba, ``Adam: A method for stochastic optimization,''
  \emph{arXiv preprint arXiv:1412.6980}, 2014.

\bibitem{duchi2011adaptive}
J.~Duchi, E.~Hazan, and Y.~Singer, ``Adaptive subgradient methods for online
  learning and stochastic optimization.'' \emph{Journal of Machine Learning
  Research}, vol.~12, no.~7, 2011.

\bibitem{ruder2016overview}
S.~Ruder, ``An overview of gradient descent optimization algorithms,''
  \emph{arXiv preprint arXiv:1609.04747}, 2016.

\bibitem{nesterov2003introductory}
Y.~Nesterov, \emph{Introductory lectures on convex optimization: A basic
  course}.\hskip 1em plus 0.5em minus 0.4em\relax Springer Science \& Business
  Media, 2003, vol.~87.

\bibitem{dauphin2014identifying}
Y.~N. Dauphin, R.~Pascanu, C.~Gulcehre, K.~Cho, S.~Ganguli, and Y.~Bengio,
  ``Identifying and attacking the saddle point problem in high-dimensional
  non-convex optimization,'' \emph{Advances in Neural Information Processing
  Systems}, vol.~27, 2014.

\bibitem{lee2016gd}
J.~D. Lee, M.~Simchowitz, M.~I. Jordan, and B.~Recht, ``Gradient descent only
  converges to minimizers,'' in \emph{29th Annual Conference on Learning
  Theory}, ser. Proceedings of Machine Learning Research, vol.~49.\hskip 1em
  plus 0.5em minus 0.4em\relax PMLR, 2016, pp. 1246--1257.

\bibitem{du2017exp}
S.~S. Du, C.~Jin, J.~D. Lee, M.~I. Jordan, B.~P\'{o}czos, and A.~Singh,
  ``Gradient descent can take exponential time to escape saddle points,'' in
  \emph{Proceedings of the 31st International Conference on Neural Information
  Processing Systems}, ser. NIPS'17, 2017, p. 1067–1077.

\bibitem{ge2015saddle}
R.~Ge, F.~Huang, C.~Jin, and Y.~Yuan, ``Escaping from saddle points - online
  stochastic gradient for tensor decomposition,'' in \emph{Proceedings of The
  28th Conference on Learning Theory, {COLT} 2015, Paris, France, July 3-6,
  2015}, ser. {JMLR} Workshop and Conference Proceedings, vol.~40.\hskip 1em
  plus 0.5em minus 0.4em\relax JMLR.org, 2015, pp. 797--842.

\bibitem{du1995minimax}
D.-Z. Du and P.~M. Pardalos, \emph{Minimax and applications}.\hskip 1em plus
  0.5em minus 0.4em\relax Springer Science \& Business Media, 1995, vol.~4.

\bibitem{goodfellow2014generative}
I.~Goodfellow, J.~Pouget-Abadie, M.~Mirza, B.~Xu, D.~Warde-Farley, S.~Ozair,
  A.~Courville, and Y.~Bengio, ``Generative adversarial nets,'' \emph{Advances
  in Neural Information Processing Systems}, vol.~27, 2014.

\bibitem{madry2017towards}
A.~Madry, A.~Makelov, L.~Schmidt, D.~Tsipras, and A.~Vladu, ``Towards deep
  learning models resistant to adversarial attacks,'' in \emph{International
  Conference on Learning Representations}, 2018.

\bibitem{berger2013statistical}
J.~O. Berger, \emph{Statistical decision theory and Bayesian analysis}.\hskip
  1em plus 0.5em minus 0.4em\relax Springer Science \& Business Media, 2013.

\bibitem{zuo2016distributed}
Z.~Zuo and L.~Tie, ``Distributed robust finite-time nonlinear consensus
  protocols for multi-agent systems,'' \emph{International Journal of Systems
  Science}, vol.~47, no.~6, pp. 1366--1375, 2016.

\bibitem{duchi2011adagrad}
J.~Duchi, E.~Hazan, and Y.~Singer, ``Adaptive subgradient methods for online
  learning and stochastic optimization,'' \emph{Journal of Machine Learning
  Research}, vol.~12, no.~61, pp. 2121--2159, 2011.

\bibitem{torch_adagrad}
``Pytorch {I}mplementation of {A}dagrad,''
  \url{https://pytorch.org/docs/stable/generated/torch.optim.Adagrad.html},
  accessed: 2023-19-16.

\bibitem{torch_rmsprop}
``Pytorch {I}mplementation of {RMS}prop,''
  \url{https://pytorch.org/docs/stable/generated/torch.optim.RMSprop.html#torch.optim.RMSprop},
  accessed: 2023-19-16.

\bibitem{muehlebach2021optimization}
M.~Muehlebach and M.~I. Jordan, ``Optimization with momentum: Dynamical,
  control-theoretic, and symplectic perspectives,'' \emph{Journal of Machine
  Learning Research}, vol.~22, no.~73, pp. 1--50, 2021.

\bibitem{garg2022}
K.~Garg, M.~Baranwal, R.~Gupta, and M.~Benosman, ``Fixed-time stable proximal
  dynamical system for solving {MVIP}s,'' \emph{IEEE Transactions on Automatic
  Control (Early Access)}, pp. 1--8, 2022.

\bibitem{suh2022continuous}
J.~J. Suh, G.~Roh, and E.~K. Ryu, ``Continuous-time analysis of accelerated
  gradient methods via conservation laws in dilated coordinate systems,'' in
  \emph{International Conference on Machine Learning}.\hskip 1em plus 0.5em
  minus 0.4em\relax PMLR, 2022, pp. 20\,640--20\,667.

\bibitem{sun2015rideablesaddle}
J.~Sun, Q.~Qu, and J.~Wright, ``When are nonconvex problems not scary?''
  \emph{arXiv preprint}, 2015.

\bibitem{lecun1998gradient}
Y.~LeCun, L.~Bottou, Y.~Bengio, and P.~Haffner, ``Gradient-based learning
  applied to document recognition,'' \emph{Proceedings of the IEEE}, vol.~86,
  no.~11, pp. 2278--2324, 1998.

\bibitem{krizhevsky2009learning}
A.~Krizhevsky, ``Learning multiple layers of features from tiny images,''
  \emph{Technical Report, University of Toronto}, 2009.

\bibitem{hinton2012neural}
\BIBentryALTinterwordspacing
G.~Hinton, N.~Srivastava, and K.~Swersky, ``Neural networks for machine
  learning, {L}ecture 6a, {O}verview of mini-batch gradient descent,'' 2012.
  [Online]. Available:
  \url{https://www.cs.toronto.edu/~bonner/courses/2016s/csc321/lectures/lec6.pdf}
\BIBentrySTDinterwordspacing

\bibitem{reddi2016stochastic}
S.~J. Reddi, A.~Hefny, S.~Sra, B.~Poczos, and A.~Smola, ``Stochastic variance
  reduction for nonconvex optimization,'' in \emph{International Conference on
  Machine Learning}.\hskip 1em plus 0.5em minus 0.4em\relax PMLR, 2016, pp.
  314--323.

\bibitem{jelassi2022adam}
\BIBentryALTinterwordspacing
S.~Jelassi, A.~Mensch, G.~Gidel, and Y.~Li, ``Adam is no better than normalized
  {SGD}: Dissecting how adaptivity improves {GAN} performance,'' 2022.
  [Online]. Available: \url{https://openreview.net/forum?id=D9SuLzhgK9}
\BIBentrySTDinterwordspacing

\bibitem{federer2014measure}
H.~Federer, \emph{Geometric Measure Theory}.\hskip 1em plus 0.5em minus
  0.4em\relax Springer, 2014.

\end{thebibliography}

\end{document}